\documentclass[twoside,11pt]{article}

\usepackage[abbrvbib]{jmlr2e}

\usepackage[linesnumbered, boxruled]{algorithm2e}
\usepackage[utf8]{inputenc} 
\usepackage{hyperref}       
\usepackage{booktabs}       
\usepackage{amsfonts}       
\usepackage{amsmath}
\usepackage{amssymb}
\usepackage{nicefrac}       
\usepackage{microtype}      

\usepackage{dblfloatfix}
\usepackage{bm,dsfont}
\usepackage{xspace} 
\def\ouralg{\textsc{Hyperband}\xspace}

\def\succhalv{\textsc{SuccessiveHalving}\xspace}

\usepackage{graphicx} 
\usepackage{subfigure} 
\usepackage{enumitem}
\usepackage{letltxmacro}
\usepackage{capt-of}
\LetLtxMacro{\originaleqref}{\eqref}
\renewcommand{\eqref}{Eq.~\originaleqref}
\newcommand{\X}{\mathcal{X}}

\newtheorem{assumption}{Assumption}

\def\E{\mathbb{E}}
\def\1{\mathbf{1}}
\def\P{\mathbb{P}}
\def\R{\mathbb{R}}

\usepackage{color}
\usepackage[normalem]{ulem}

\newcommand{\Rmax}{R}

\SetCommentSty{mycommfont}
\usepackage{lastpage}
\jmlrheading{18}{2018}{1-\pageref{LastPage}}{11/16; Revised 12/17}{4/18}{16-558}{Lisha Li, Kevin Jamieson, Giulia DeSalvo, Afshin Rostamizadeh and Ameet Talwalkar}
\ShortHeadings{Bandit-Based Approach to Hyperparameter Optimization}{Li, Jamieson, DeSalvo, Rostamizadeh and Talwalkar}
\firstpageno{1}

\begin{document}
\title{\ouralg: A Novel Bandit-Based Approach to Hyperparameter Optimization}

\author{\name Lisha Li  \email lishal@cs.cmu.edu \\
  \addr Carnegie Mellon University, Pittsburgh, PA 15213 \\
  \name{Kevin Jamieson} \email{jamieson@cs.washington.edu} \\
  \addr University of Washington, Seattle, WA 98195\\
  \name{Giulia DeSalvo} \email{giuliad@google.com} \\
  \addr Google Research, New York, NY 10011\\
  \name{Afshin Rostamizadeh} \email{rostami@google.com}  \\
  \addr Google Research, New York, NY 10011\\
  \name{Ameet Talwalkar} \email{talwalkar@cmu.edu}\\
  \addr Carnegie Mellon University, Pittsburgh, PA 15213\\
  \addr Determined AI 
}

\editor{Nando de Freitas}

\maketitle

\begin{abstract}%
Performance of machine learning algorithms depends critically on identifying a
good set of hyperparameters. While recent approaches use Bayesian optimization to adaptively select configurations, we focus on speeding up random search through adaptive resource allocation and early-stopping.  We
formulate hyperparameter optimization as a pure-exploration non-stochastic
infinite-armed bandit problem where a predefined resource like iterations, data samples, or features is allocated to randomly sampled configurations.  We introduce a novel algorithm, \ouralg , for this framework and analyze its theoretical properties, providing several desirable guarantees.  Furthermore, we compare \ouralg with
popular Bayesian optimization methods on a suite of hyperparameter optimization problems. 
We observe that \ouralg can provide over an order-of-magnitude speedup over our competitor set on a variety of deep-learning and kernel-based learning problems.  

\end{abstract}
\begin{keywords}
hyperparameter optimization, model selection, infinite-armed bandits, online optimization, deep learning
\end{keywords}

\section{Introduction}\label{sec:intro}

In recent years, machine learning models have exploded in complexity and expressibility at the price of staggering computational costs. 
Moreover, the growing number of tuning parameters associated with these models are difficult to set by standard optimization techniques.  
These ``hyperparameters'' are inputs to a machine learning algorithm that govern how the algorithm's performance generalizes to new, unseen data; examples of hyperparameters include those that impact model architecture, amount of regularization, and learning rates.
The quality of a predictive model critically depends on its hyperparameter configuration, but it is poorly understood how these hyperparameters interact with each other to affect the resulting model. Consequently, practitioners often default to brute-force methods like random search and grid search \citep{Bergstra2012}.

In an effort to develop more efficient search methods, the problem of hyperparameter optimization has recently been dominated by  {\em Bayesian optimization} methods ~\citep{Snoek2012, Hutter2011, Bergstra2011} that focus on optimizing hyperparameter \emph{configuration selection}.  These methods  aim to identify good
configurations more quickly than standard baselines like random search by selecting configurations in
an adaptive manner; see Figure~\ref{fig:config_selection}. 
Existing empirical evidence suggests that these methods
outperform random search~\citep{Thornton2013,Eggensperger2013,Snoek2015}. However, these methods tackle the fundamentally challenging problem of simultaneously fitting and optimizing a high-dimensional, non-convex function with unknown smoothness, and possibly noisy evaluations.  
\begin{figure}[t]
\centering
\subfigure[Configuration Selection]{%
\label{fig:config_selection}
\includegraphics[height=4cm,trim=10 10 10 30]{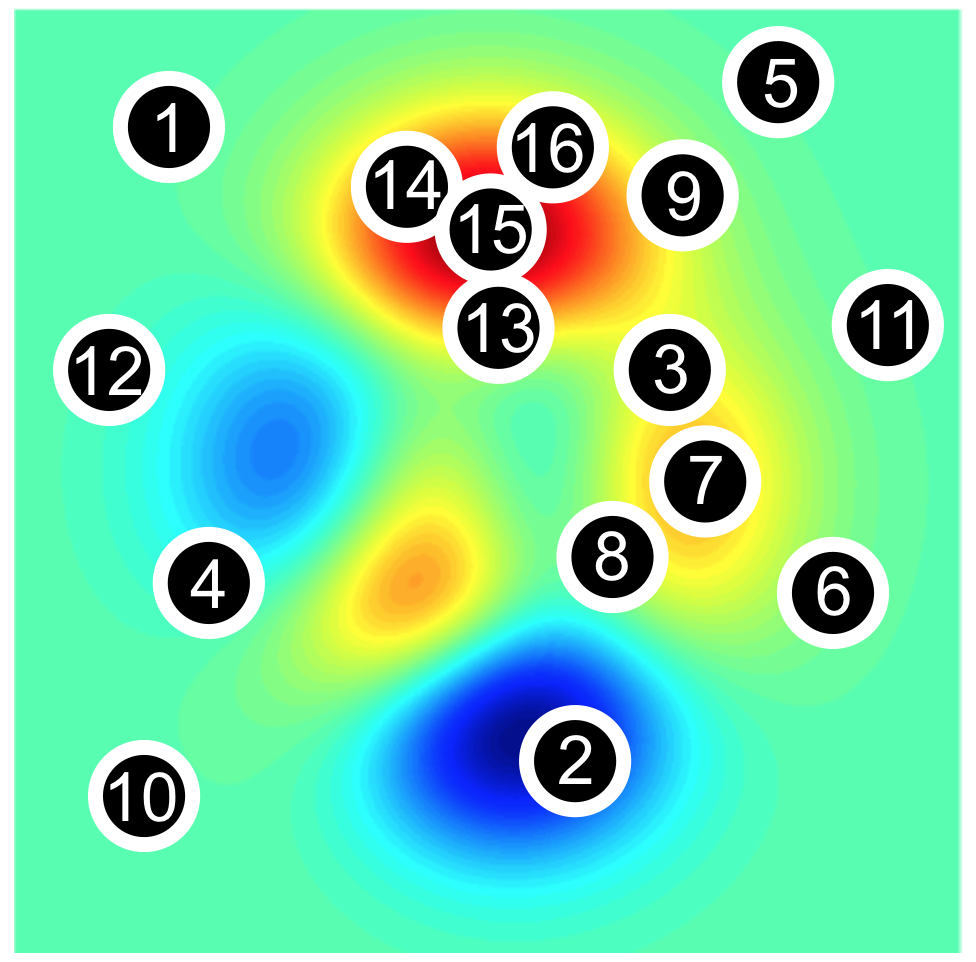}
}\qquad 
\subfigure[Configuration Evaluation]{%
\label{fig:config_eval}
\includegraphics[height=4cm,trim=10 10 10 30]{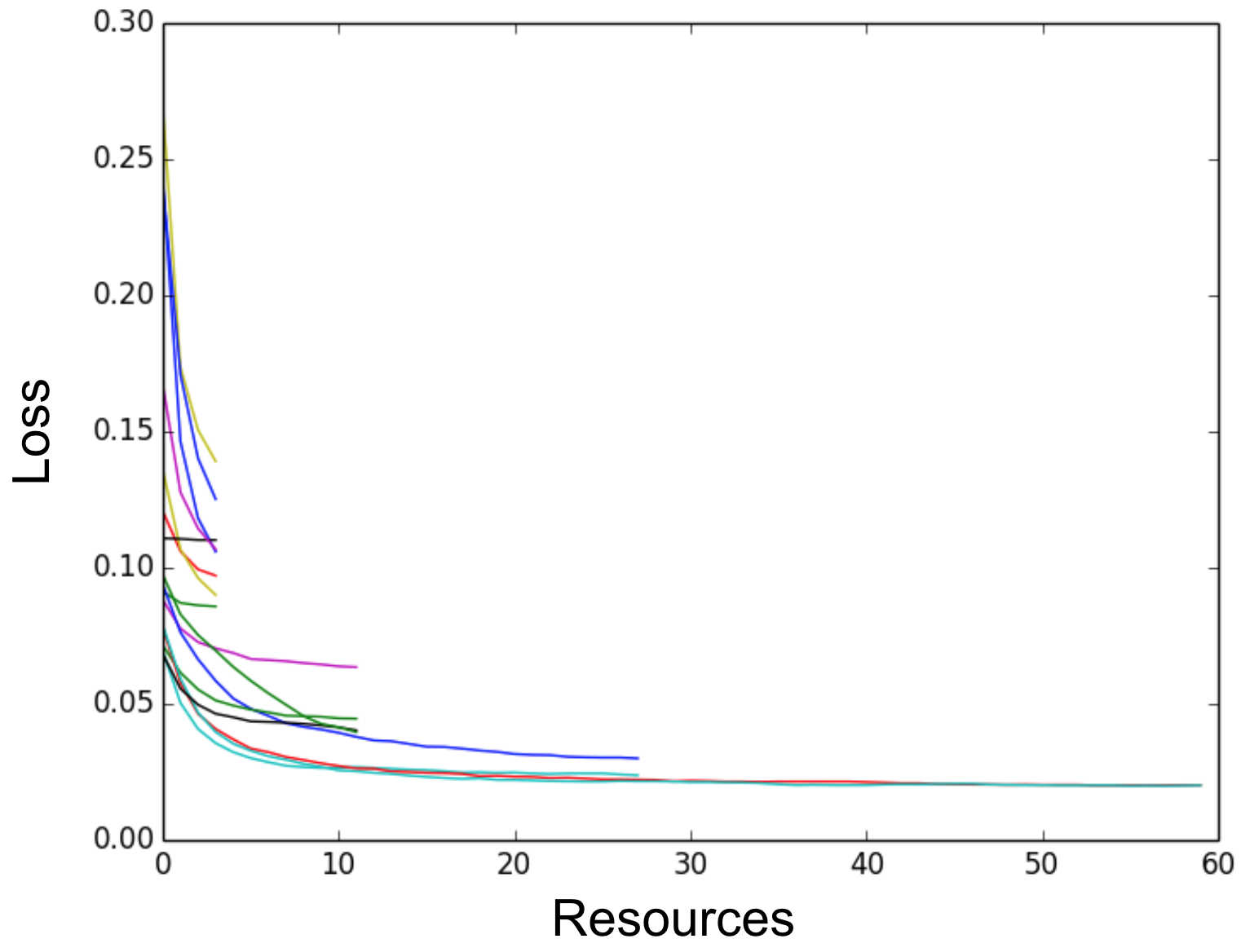}
}
\caption{(a) The heatmap shows the validation error over a two-dimensional search space with red corresponding to areas with lower validation error.  Configuration selection methods adaptively choose new configurations to train, proceeding in a sequential manner as indicated by the numbers.  (b) The plot shows the validation error as a function of the resources allocated to each configuration (i.e.\ each line in the plot).  Configuration evaluation methods allocate more resources to promising configurations.} \label{fig:adaptive_comparison}
\end{figure}

An orthogonal approach to hyperparameter optimization focuses on speeding up \emph{configuration evaluation}; see Figure~\ref{fig:config_eval}.
These approaches are adaptive in computation,
 allocating more resources to promising hyperparameter configurations while
quickly eliminating poor ones.  Resources can take various forms, including size of training set, number of features, or number of iterations for iterative algorithms.
By adaptively allocating resources, these approaches aim to examine orders-of-magnitude more hyperparameter
configurations than approaches that uniformly train all configurations to completion, thereby
quickly identifying good hyperparameters. 
While there are methods that combine Bayesian optimization with adaptive resource allocation \citep{Swersky2013Multi,swersky2014freeze,earlystopping2015,fabolas2017}, we focus on speeding up random search as it offers a simple 
and theoretically principled launching point \citep{Bergstra2012}.\footnote{Random search will asymptotically converge to the optimal configuration, regardless of the smoothness or structure of the function being optimized, by a simple covering argument.  While the rate of convergence for random search depends on the smoothness and is exponential in the number of dimensions in the search space, the same is true for Bayesian optimization methods without additional structural assumptions \citep{kandasamy2015}. } 

We develop a novel configuration evaluation approach by formulating hyperparameter optimization as a pure-exploration adaptive resource allocation problem addressing how to allocate resources among randomly sampled hyperparameter configurations.\footnote{A preliminary version of this work appeared in \citet{hyperband}.  We extend the previous paper with a thorough theoretical analysis of \ouralg; an infinite horizon version of the algorithm with application to stochastic infinite-armed bandits; additional intuition and discussion of \ouralg to facilitate its use in practice; and additional results on a collection of 117 multistage model selection tasks.}
Our procedure, \ouralg, relies on a principled early-stopping strategy to allocate resources, allowing it to evaluate orders-of-magnitude more configurations than black-box procedures like Bayesian optimization methods.  
\ouralg is a general-purpose technique that makes minimal assumptions unlike prior configuration
 evaluation approaches ~\citep{earlystopping2015,swersky2014freeze,GyorgyK11restart,agarwal2012oracle,Sparks2015,JamiesonTalwalkar2015}.

Our theoretical analysis demonstrates the ability of \ouralg to adapt to unknown convergence rates and to the behavior of  validation losses as a function of the hyperparameters.
In addition, \ouralg is $5\times$ to $30\times$  faster than popular Bayesian optimization algorithms on a variety of deep-learning and kernel-based learning problems. 
A theoretical contribution of this work is the introduction of the pure-exploration, infinite-armed bandit problem in the non-stochastic setting, for which \ouralg is one solution.
When \ouralg is applied to the special-case stochastic setting, we show that the algorithm comes within $\log$ factors of known lower bounds in both the infinite \citep{carpentier2015simple} and finite $K$-armed bandit settings \citep{kaufmann2015complexity}.

The paper is organized as follows.  Section~\ref{sec:related} summarizes related work in two areas: (1) hyperparameter optimization, and (2) pure-exploration bandit problems.  Section~\ref{sec:algorithm} describes \ouralg and provides intuition for the algorithm through a detailed example.  In Section~\ref{sec:experiments}, we present a wide range of empirical results comparing \ouralg with state-of-the-art competitors.  Section~\ref{sec:theory} frames the hyperparameter optimization problem as an infinite-armed bandit problem and summarizes the theoretical results for \ouralg.    Finally, Section~\ref{sec:extensions} discusses possible extensions of \ouralg.


\section{Related Work}\label{sec:related}
In Section~\ref{sec:intro}, we briefly discussed related work in the hyperparameter optimization literature. Here, we provide a more thorough coverage of the prior work, and also summarize significant related work on bandit problems. 
\subsection{Hyperparameter Optimization}\label{ssec:related_hyper}
Bayesian optimization techniques model the
conditional probability $p( y| \lambda)  $  of a configuration's performance on an evaluation metric $y$ (i.e., test accuracy), given a
set of hyperparameters $\lambda$.   Sequential Model-based Algorithm Configuration
(SMAC), Tree-structure Parzen Estimator (TPE), and Spearmint are three well-established methods \citep{Feurer2014}.  SMAC  uses random forests to model  $p( y| \lambda) $  as a Gaussian distribution  \citep{Hutter2011}.  TPE is a
non-standard Bayesian optimization algorithm based on tree-structured Parzen
density estimators \citep{Bergstra2011}.  Lastly, Spearmint uses Gaussian processes (GP)
to model $p( y | \lambda) $ and performs slice sampling over the GP’s
hyperparameters \citep{Snoek2012}.  

Previous work compared the relative performance of these Bayesian searchers
\citep{Thornton2013,Eggensperger2013,Bergstra2011,Snoek2012,Feurer2014,Feurer2015}. 
An extensive survey of these three methods by \citet{Eggensperger2013} introduced a benchmark library for hyperparameter optimization called HPOlib, which we use for our experiments.
\citet{Bergstra2011} and \citet{Thornton2013} showed Bayesian optimization methods empirically outperform
random search on a few benchmark tasks.  However, for high-dimensional problems, standard Bayesian optimization methods perform similarly to random search \citep{wang2013}.  Recent methods specifically designed for high-dimensional problems assume a lower effective dimension for the problem \citep{wang2013} or an additive decomposition for the target function \citep{kandasamy2015}.  However, as can be expected, the performance of these methods is sensitive to required inputs; i.e.\ the effective dimension \citep{wang2013} or the number of additive components \citep{kandasamy2015}.

Gaussian processes have also been studied in the bandit setting using confidence bound acquisition functions (GP-UCB), with associated sublinear regret bounds \citep{srinivas2010,grunewalder2010}.  \citet{wang2016} improved upon GP-UCB by removing the need to tune a parameter that controls exploration and exploitation.  \citet{contal2014} derived a tighter regret bound than that for GP-UCB by using a mutual information acquisition function.  However, \citet{vander2011} showed that the learning rate of GPs are sensitive to the definition of the prior through an example with a poor prior where the learning rate degraded from polynomial to logarithmic in the number of observations $n$.  Additionally, without structural assumptions on the covariance matrix of the GP, fitting the posterior is $O(n^3)$ \citep{wilson2015}.  Hence, \citet{dngo2015} and \cite{bohamiann2016} proposed using Bayesian neural networks, which scale linearly with $n$, to model the posterior.

Adaptive configuration evaluation is not a new idea.  \citet{Maron1997} and \citet{mnih2008} considered a setting where the training time is relatively inexpensive (e.g., $k$-nearest-neighbor classification) and evaluation on a large validation set is accelerated by evaluating on an increasing subset of the validation set, stopping early configurations that are performing poorly.  Since subsets of the validation set provide unbiased estimates of its expected performance, this is an instance of the \emph{stochastic} best-arm identification problem for multi-armed bandits \citep[see the work by][for a brief survey]{jamieson2014best}.

In contrast, we address a setting where the evaluation time is relatively inexpensive and the goal is to early-stop long-running training procedures by evaluating partially trained models on the full validation set.  Previous approaches in this setting either require strong assumptions or use heuristics to perform adaptive resource allocation. 
\citet{GyorgyK11restart} and \citet{agarwal2012oracle} made parametric assumptions on the convergence behavior of training
algorithms, providing theoretical performance guarantees under these assumptions.
Unfortunately, these assumptions are often hard to verify, and
empirical performance can drastically suffer when they are violated.
 \citet{Krueger2015}
proposed a heuristic based on sequential analysis to determine stopping
times for training configurations on increasing subsets of the data. However,
the theoretical correctness and empirical performance of this method are highly dependent on a user-defined
``safety zone.''

Several hybrid methods  combining adaptive configuration selection and evaluation  have also been introduced~\citep{Swersky2013Multi,swersky2014freeze,earlystopping2015,multifid2016,fabolas2017,vizier2017}.  The algorithm proposed by \citet{Swersky2013Multi} uses a GP to learn correlation between related tasks and requires the subtasks as input, but efficient subtasks with high informativeness for the target task are unknown without prior knowledge. Similar to the work by \citet{Swersky2013Multi}, \citet{fabolas2017} modeled the conditional validation error as a Gaussian process using a kernel that
captures the covariance with downsampling rate to allow for adaptive evaluation.
\citet{swersky2014freeze}, \citet{earlystopping2015}, and \cite{fabolas2017} made parametric assumptions
on the convergence of learning curves to perform early-stopping.  In contrast, \citet{vizier2017} devised an early-stopping rule based on predicted performance from a nonparametric GP model with a kernel designed to measure the similarity between performance curves.  Finally, \citet{multifid2016} extended GP-UCB to allow for adaptive configuration
evaluation by defining subtasks that monotonically improve with more resources.  

In another line of work, \citet{Sparks2015} proposed a halving style bandit 
algorithm that did not require explicit convergence behavior, and
\citet{JamiesonTalwalkar2015} analyzed a similar algorithm originally proposed by \cite{karnin2013almost} for a different setting, providing
theoretical guarantees and encouraging empirical
results.   Unfortunately, these halving style algorithms suffer from the ``$n$ versus $B/n$'' problem, which we will discuss in Section~\ref{ssec:succhalv_nvb}. \ouralg addresses this issue and provides a robust, theoretically principled early-stopping algorithm for hyperparameter optimization.  

We note that \ouralg can be combined with any hyperparameter sampling approach and does not depend on random sampling; the theoretical results only assume the validation losses of sampled hyperparameter configurations are drawn from some stationary distribution.  In fact, subsequent to our submission, \citet{klein2017} combined adaptive configuration selection with \ouralg by using a Bayesian neural network to model learning curves and only selecting configurations with high predicted performance to input into \ouralg.  

\subsection{Bandit Problems}
Pure exploration bandit problems aim to minimize the simple regret, defined as the distance from the optimal solution, as quickly as possible in any given setting.  
  The
  pure-exploration multi-armed bandit problem has a long history in the stochastic setting
\citep{even2006action,bubeck2009pure}, and was recently extended to the
non-stochastic setting by \citet{JamiesonTalwalkar2015}.  Relatedly, the stochastic pure-exploration infinite-armed bandit problem was studied by \citet{carpentier2015simple}, where a pull of each arm $i$ yields an i.i.d. sample
in $[0,1]$ with expectation $\nu_i$, where $\nu_i$ is a loss drawn from a distribution with cumulative distribution function, $F$. Of course, the value of $\nu_i$ is unknown
to the player, so the only way to infer its value is to pull arm $i$ many times.
\citet{carpentier2015simple} proposed an anytime algorithm, and derived
a tight (up to polylog factors) upper bound on its error assuming what we will refer to as the $\beta$-parameterization of $F$ described in Section~\ref{sec:param}.  
However, their algorithm was derived specifically for the $\beta$-parameterization of $F$, and furthermore, they must estimate $\beta$ before running the
algorithm, limiting the algorithm's practical applicability.
Also, the algorithm assumes stochastic losses from the arms and thus the convergence behavior is known; consequently, it does not apply in our hyperparameter optimization
setting.\footnote{See the work by \citet{JamiesonTalwalkar2015} for detailed discussion motivating the
non-stochastic setting for hyperparameter optimization.} 
Two related lines of work that both make use of an underlying metric space are 
Gaussian process optimization \citep{srinivas2010} and $X$-armed bandits
\citep{bubeck2011x}, or bandits defined over a metric space. However,
these works either assume stochastic rewards or need to know something about the underlying function (e.g. an
appropriate kernel or level of smoothness).

In contrast, \ouralg 
is devised
for the non-stochastic setting and automatically adapts to unknown $F$ without
  making any parametric assumptions. Hence, we believe our work to be a generally
   applicable pure exploration algorithm for infinite-armed
  bandits.  To the best of our knowledge, this is also the first work to test out such an algorithm on a real application. 


\section{\ouralg Algorithm} \label{sec:algorithm}
In this section, we present the \ouralg algorithm.  We provide intuition for the algorithm, highlight the main ideas via a simple example that uses iterations as the adaptively allocated resource, 
and present a few guidelines on how to deploy \ouralg in practice.  

\subsection{Successive Halving}\label{ssec:succhalv_nvb}
\ouralg extends the \succhalv algorithm proposed for hyperparameter optimization by \citet{JamiesonTalwalkar2015} and calls it as a subroutine.  
The idea behind the original \succhalv algorithm follows directly from its name: uniformly allocate a budget to a set of hyperparameter configurations, evaluate the performance of all configurations, throw out the worst half, and repeat until one configuration remains.  
The algorithm allocates exponentially more resources to more promising configurations. 
Unfortunately, \succhalv requires the number of configurations $n$ as an input to the algorithm.  
Given some finite budget $B$ (e.g., an hour of training time to choose a
hyperparameter configuration), $B/n$ resources are allocated on average across the configurations.
However, for a fixed $B$, it is not clear a priori whether we should  
(a) consider many configurations (large $n$)
with a small average training time; or (b) consider a small
number of configurations (small $n$) with longer average training times.

\begin{figure}[t]
\centering
\includegraphics[height=4cm,trim=10 10 10 0]{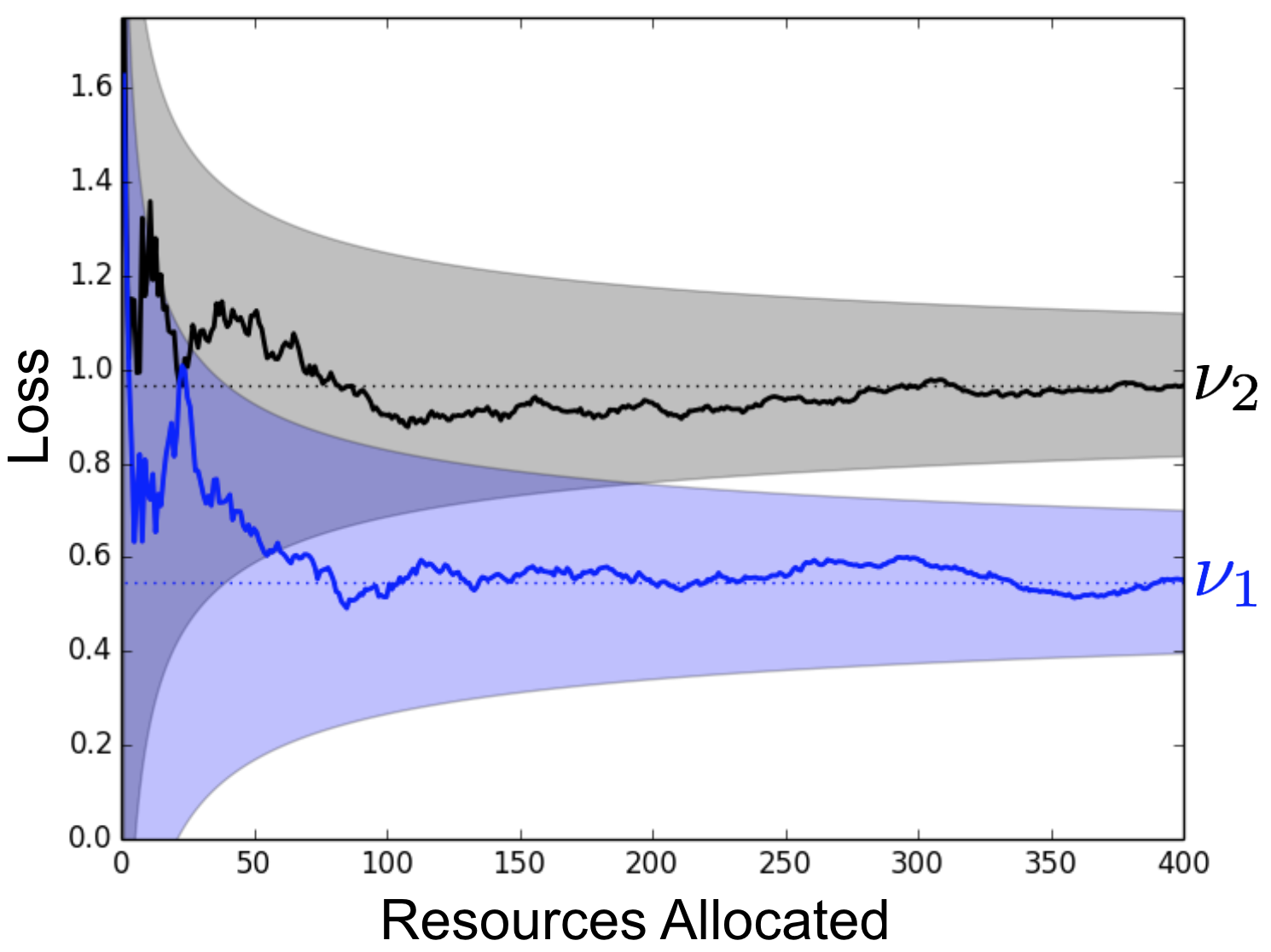}
\caption{The validation loss as a function of total resources allocated for two configurations is shown.  $\nu_1$ and $\nu_2$ represent the terminal validation losses at convergence.  The shaded areas bound the maximum distance of the intermediate losses from the terminal validation loss and monotonically decrease with the resource.} \label{fig:envelope}
\end{figure}

We use a simple example to better understand this tradeoff.  Figure~\ref{fig:envelope} shows the validation loss as a function of total resources allocated for two configurations with terminal validation losses $\nu_1$ and $\nu_2$.  The shaded areas bound the maximum deviation of the intermediate losses from the terminal validation loss and will be referred to as ``envelope'' functions.\footnote{These envelope functions are guaranteed to exist; see discussion in Section~\ref{sec:theory_assumptions} where we formally define these envelope (or $\gamma$) functions.}  It is possible to distinguish between the two configurations when the envelopes no longer overlap.  Simple arithmetic shows that this happens when the width of the envelopes is less than $\nu_2-\nu_1$, i.e., when the intermediate losses are guaranteed to be less than $\frac{\nu_2-\nu_1}{2}$ away from the terminal losses.  There are two takeaways from this observation: more resources are needed to differentiate between the two configurations when either (1) the envelope functions are wider or (2) the terminal losses are closer together.

However, in practice, the optimal allocation strategy is unknown because we do not have knowledge of the envelope functions nor the distribution of terminal losses.  Hence, if more resources are required
before configurations can differentiate themselves in terms of quality (e.g., if an iterative training method converges very slowly for a given data set or if randomly selected hyperparameter configurations perform similarly well),
then it would be reasonable to work with a small number of configurations. In contrast,
if the quality of a configuration is typically revealed after a small number of resources (e.g., if iterative training methods converge very
  quickly for a given data set or if randomly selected hyperparameter
  configurations are of low-quality with high probability), then $n$ is the
  bottleneck and we should choose $n$ to be large. 

Certainly, if meta-data or previous experience suggests that a certain tradeoff is likely to work well in practice, one should exploit that information and allocate the majority of resources to that tradeoff.  However, without this supplementary information,  practitioners are forced to make this tradeoff, severely hindering the applicability of existing configuration evaluation methods. 
  
\subsection{\ouralg}\label{ssec:ouralg}
\ouralg, shown in Algorithm \ref{alg:hyperband}, addresses this ``$n$ versus $B/n$'' problem by considering several possible values of $n$ for a fixed $B$, in essence performing a grid search over feasible value of $n$.
Associated with each value of $n$ is a minimum resource $r$ that is allocated to all configurations before some are discarded; a larger value of $n$ corresponds to a smaller $r$ and hence more aggressive early-stopping. 
  There are two components to \ouralg; (1) the inner loop invokes \succhalv for fixed values of $n$ and $r$ (lines 3--9) and (2) the outer loop iterates over different values of $n$ and $r$ (lines 1--2).
We will refer to each such run of \succhalv within \ouralg as a ``bracket.''
Each bracket is designed to use approximately $B$ total resources and corresponds to a different tradeoff between $n$ and $B/n$.
Hence, a single execution of \ouralg takes a finite budget of $(s_{max}+1) B$; we recommend repeating it indefinitely.    

\ouralg requires two inputs (1) $\Rmax$, the maximum amount of resource that can be allocated to a single configuration, and (2) $\eta$, an input that controls the proportion of configurations discarded in each round of \succhalv.  
The two inputs dictate how many different brackets are considered; specifically, $s_{\max}+1$ different values for $n$ are considered with  $s_{\max}=\lfloor\log_\eta(\Rmax)\rfloor$.    \ouralg begins with the most aggressive bracket $s=s_{\max}$, which sets $n$ to maximize exploration, subject to the constraint that at least one configuration is allocated $\Rmax$ resources. Each subsequent bracket reduces $n$ by a factor of approximately $\eta$ until the final bracket, $s=0$, in which every configuration is allocated $\Rmax$ resources (this bracket simply performs classical random search).  Hence, \ouralg performs a geometric search in the average budget per
configuration and removes the need to select $n$ for a fixed budget at the cost of approximately $s_{\max}+1$ times more work than running \succhalv for a single value of $n$.  By doing so, \ouralg is able to exploit situations in which adaptive allocation works well, while protecting itself in situations where more conservative allocations are required.

\begin{algorithm}[t]
	\SetKwInOut{Input}{input}
	\SetKwInOut{Init}{initialization}
	\Input{$\Rmax$, $\eta$ (default $\eta = 3$)} 
	\Init{$s_{\max} = \lfloor \log_\eta(\Rmax)\rfloor$, $B=(s_{\max}+1)\Rmax$}
	\For{$s \in \{s_{\max}, s_{\max} - 1, \ldots, 0\}$}{
		$n = \lceil \frac{B}{\Rmax} \frac{\eta^s}{(s+1)}  \rceil , \quad \quad r=\Rmax\eta^{-s}$\\
		\tcp{begin \succhalv with $(n,r)$ inner loop}
		$T=$\texttt{get\_hyperparameter\_configuration}$(n)$\\
		\For{$i\in\{0,\ldots,s\}$}{
			$n_i=\lfloor n\eta^{-i} \rfloor$\\
			$r_i=r\eta^{i}$\\
			$L=\{$\texttt{run\_then\_return\_val\_loss}$(t,r_i):t\in T\}$\\
			$T=$\texttt{top\_k}$(T,L, \lfloor n_i/\eta\rfloor )$
		}
	}
	\Return{Configuration with the smallest intermediate loss seen so far.}
	\\
	\caption{\ouralg algorithm for hyperparameter optimization.}
	\label{alg:hyperband}
\end{algorithm}

\ouralg requires the following methods to be defined for any given learning problem: 
\begin{itemize}
\item \texttt{get\_hyperparameter\_configuration($n$)} -- a function that returns a set of $n$ i.i.d. samples from some distribution defined over the hyperparameter configuration space.  In this work, we assume uniformly sampling of hyperparameters from a predefined space (i.e., hypercube with min and max bounds for each hyperparameter), which immediately yields consistency guarantees.  However, the more aligned the distribution is towards high quality hyperparameters (i.e., a useful prior), the better \ouralg will perform (see Section~\ref{sec:extensions} for further discussion).
\item \texttt{run\_then\_return\_val\_loss($t$, $r$)} -- a function that takes a hyperparameter configuration $t$ and resource allocation $r$ as input and returns the validation loss after training the configuration for the allocated resources.  
\item \texttt{top\_k(configs, losses, $k$)} -- a function that takes a set of configurations as well as their associated losses and returns the top $k$ performing configurations.
\end{itemize}

\subsection{Example Application with Iterations as a Resource: LeNet}\label{ssec:lenet}
We next present a concrete example to provide further intuition about \ouralg.  We work with the MNIST data set and optimize hyperparameters for the LeNet convolutional neural network trained using mini-batch stochastic gradient descent (SGD).\footnote{Code and description of algorithm used is available at \url{http://deeplearning.net/tutorial/lenet.html}.}  Our
 search space includes learning rate, batch size, and number of kernels for the two layers of the network as hyperparameters (details are shown in Table~\ref{tab:lenet} in Appendix~\ref{add_results}).

We define the resource allocated to each configuration to be number of iterations of SGD, with 
 one unit of resource corresponding to one epoch, i.e., a full pass over the data set.  We set $\Rmax$ to 81 and
use the default value of $\eta=3$,  resulting in $s_{\max}=4$ and thus 5 brackets of \succhalv  with different tradeoffs between $n$ and $B/n$. The resources allocated within each bracket are displayed in Table~\ref{table:brackets}.

\begin{table}[t]
\begin{center}
\begin{tabular}{|l|ll|ll|ll|ll|ll|}
	\hline
 & \multicolumn{2}{l|}{$s=4$} & \multicolumn{2}{l|}{$s=3$} & \multicolumn{2}{l|}{$s=2$} & \multicolumn{2}{l|}{$s=1$} & \multicolumn{2}{l|}{$s=0$} \\
$i$ & $n_i$ &  $r_i$ &  $n_i$ &  $r_i$ &  $n_i$ &  $r_i$ &  $n_i$ &  $r_i$ &  $n_i$ &  $r_i$  \\
\hline
0 & 81 & 1 & 27 & 3 & 9 & 9 & 6 & 27 & 5 & 81 \\         
1 & 27 & 3 & 9 &  9 & 3 & 27 & 2 & 81 & & \\        
2 & 9 & 9 & 3 & 27 & 1 & 81 & & & &\\       
3 & 3 & 27 & 1 & 81 & & & & & & \\
4 & 1  & 81 & & & & & & & & \\ \hline
\end{tabular}
\caption{The values of $n_i$ and $r_i$ for the brackets of \ouralg corresponding to various values of $s$, when $\Rmax = 81$ and $\eta = 3$.}
\label{table:brackets}
\end{center}
\end{table}

\begin{figure}[h!]
\centering
\includegraphics[trim=0 20 0 20, height=5cm]{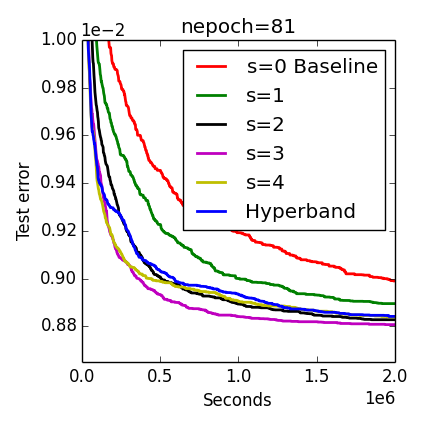}
\captionof{figure}{Performance of individual brackets $s$ and \ouralg.}
\label{fig:s_comparison}
\end{figure}

Figure~\ref{fig:s_comparison} shows an empirical comparison of the average test error across 70 trials of the individual brackets of \ouralg run separately as well as standard \ouralg.
In practice, we do not know a priori which bracket $s \in \{0,\dots,4\}$ will be most effective in identifying good hyperparameters, and in this case neither the most ($s=4$) nor least aggressive ($s=0$) setting is optimal.  
However, we note that \ouralg does nearly as well as the optimal bracket ($s=3$) and outperforms the baseline uniform allocation (i.e., random search), which is equivalent to bracket $s=0$.

\subsection{Different Types of Resources}
While the previous example focused on iterations as the resource, \ouralg naturally generalizes to various types of resources:
\begin{itemize}
\item {\bf Time} -- Early-stopping in terms of time can be preferred when various hyperparameter configurations differ in training time and the practitioner's chief goal is to find a good hyperparameter setting in a fixed wall-clock time.  
For instance, training time could be used as a resource to quickly terminate straggler jobs in distributed computation environments.

\item {\bf Data Set Subsampling} -- Here we consider the setting of a black-box batch training algorithm that takes a data set as input and outputs a model. In this setting, we treat the resource as the size of a random subset of the data set with $\Rmax$ corresponding to the full data set size. 
Subsampling data set sizes using \ouralg, especially for problems with super-linear training times like kernel methods, can provide substantial speedups.
 
\item {\bf Feature Subsampling} -- Random features or Nystr\"om-like methods are popular methods for approximating kernels for machine learning applications \citep{randomfeatures2007}. In image processing, especially deep-learning applications, filters are usually sampled randomly, with the number of filters having an impact on the performance. Downsampling the number of features is a common tool used when hand-tuning hyperparameters; \ouralg can formalize this heuristic.
\end{itemize}

\subsection{Setting $\Rmax$}
The resource $\Rmax$ and $\eta$ (which we address next) are the only required inputs to \ouralg. As mentioned in Section~\ref{ssec:ouralg}, $\Rmax$ represents the maximum amount of resources that can be allocated to any given configuration.  In most cases, there is a natural upper bound on the maximum budget per configuration that is often dictated by the resource type (e.g., training set size for data set downsampling; limitations based on memory constraint for feature downsampling; rule of thumb regarding number of epochs when iteratively training neural networks). 
If there is a range of possible values for $\Rmax$, a smaller $\Rmax$ will give a result faster (since the budget $B$ for each bracket is a multiple of $\Rmax$), but a larger $\Rmax$ will give a better guarantee of successfully differentiating between the configurations.  

Moreover, for settings in which either $\Rmax$ is unknown or not desired, we provide an infinite horizon version of \ouralg in Section~\ref{sec:theory}.  
This version of the algorithm doubles the budget over time, $B \in \{2, 4, 8, 16, \ldots\}$, and for each $B$, tries all possible values of $n \in \big\{ 2^k : k \in \{1,\ldots,\log_2(B)\} \big\}$. 
For each combination of $B$ and $n$, the algorithm runs an instance of the (infinite horizon) \succhalv algorithm, which implicitly sets $\Rmax=\frac{B}{2\log_2(n)}$, thereby growing $\Rmax$ as $B$ increases.
The main difference between the infinite horizon algorithm and Algorithm~\ref{alg:hyperband} is that the number of unique brackets grows over time instead of staying constant with each outer loop.
We will analyze this version of \ouralg in more detail in Section~\ref{sec:theory} and use it as the launching point for the theoretical analysis of standard (finite horizon) \ouralg.

Note that $\Rmax$ is also the number of configurations evaluated in the bracket that performs the most exploration, i.e $s=s_{\max}$.  In practice one may want $n \leq n_{\max}$ to limit overhead associated with training many configurations on a small budget, i.e., costs associated with initialization, loading a model, and validation.  In this case, set $s_{\max} = \lfloor \log_\eta(n_{\max})\rfloor$.  Alternatively, one can redefine one unit of resource so that $\Rmax$ is artificially smaller (i.e., if the desired maximum iteration is 100k, defining one unit of resource to be 100 iterations will give $\Rmax=1,000$, whereas defining one unit to be 1k iterations will give $\Rmax=100$).  Thus, one unit of resource can be interpreted as the minimum desired resource and $\Rmax$ as the ratio between maximum resource and minimum resource.  

\subsection{Setting $\eta$}
\label{sec:eta} 
 The value of $\eta$ is a knob that can be tuned based on {\em practical} user constraints.  Larger values of $\eta$ correspond to more aggressive elimination schedules and thus fewer rounds of elimination; specifically, each round retains $1/\eta$ configurations for a total of $\lfloor\log_\eta(n)\rfloor+1$ rounds of elimination with $n$ configurations.  
If one wishes to receive a result faster at the cost of a sub-optimal asymptotic constant, one can increase $\eta$ to reduce the budget per bracket $B=(\lfloor\log_\eta(\Rmax)\rfloor+1)\Rmax$.  We stress that results are not very sensitive to the choice of $\eta$. If our theoretical bounds are optimized (see Section~\ref{sec:theory}), they suggest choosing $\eta = e \approx 2.718$, but in practice we suggest taking $\eta$ to be equal to 3 or 4.  

Tuning $\eta$ will also change the number of brackets and consequently the number of different tradeoffs that \ouralg tries.  
 Usually, the possible range of brackets is fairly constrained, since the number of brackets is logarithmic in $\Rmax$; namely, there are $(\lfloor\log_\eta(\Rmax)\rfloor+1)=s_{\max}+1$ brackets.  For our experiments in Section~\ref{sec:experiments}, we chose $\eta$ to provide 5 brackets for the specified $R$; for most problems, 5 is a reasonable number of $n$ versus $B/n$ tradeoffs to explore.  However, for large $\Rmax$, using $\eta=3$ or 4 can give more brackets than desired.  The number of brackets can be controlled in a few ways.  First, as mentioned in the previous section, if $\Rmax$ is too large and overhead is an issue, then one may want to control the overhead by limiting the maximum number of configurations to $n_{\max}$, thereby also limiting $s_{\max}$. If overhead is not a concern and aggressive exploration is desired, one can (1) increase $\eta$ to reduce the number of brackets while maintaining $\Rmax$ as the maximum number of configurations in the most exploratory bracket, or (2) still use $\eta=3$ or 4 but only try brackets that do a baseline level of exploration, i.e., set $n_{\min}$ and only try brackets from $s_{\max}$ to $s=\lfloor\log_\eta(n_{\min})\rfloor$.  For computationally intensive problems that have long training times and high-dimensional search spaces, we recommend the latter.  Intuitively, if the number of configurations that can be trained to completion (i.e., trained using $\Rmax$ resources) in a reasonable amount of time is on the order of the dimension of the search space and not exponential in the dimension, then it will be impossible to find a good configuration without using an aggressive exploratory tradeoff between $n$ and $B/n$.

\subsection{Overview of Theoretical Results}
\label{sec:theory_overview} 
The theoretical properties of \ouralg are best demonstrated through an example.  Suppose there are $n$ configurations, each with a given terminal validation error $\nu_i$ for $i=1,\dots,n$.  
Without loss of generality, index the configurations by performance so that $\nu_1$ corresponds to the best performing configuration, $\nu_2$ to the second best, and so on.  
Now consider the task of identifying the best configuration.  The optimal strategy would allocate to each configuration $i$ the minimum resource required to distinguish it from $\nu_1$, i.e., enough so that the envelope functions (see Figure~\ref{fig:envelope})  bound the intermediate loss to be less than $\frac{\nu_i-\nu_1}{2}$ away from the terminal value.  
In contrast, the naive uniform allocation strategy, which allocates $B/n$ to each configuration, has to allocate to every configuration the maximum resource required to distinguish any arm $\nu_i$ from $\nu_1$.  
Remarkably, the budget required by \succhalv is only a small factor of the optimal because it capitalizes on configurations that are easy to distinguish from $\nu_1$.  

The relative size of the budget required for uniform allocation and \succhalv depends on the envelope functions bounding deviation from terminal losses  as well as the distribution from which $\nu_i$'s are drawn.  The budget required for \succhalv is smaller when the optimal $n$ versus $B/n$ tradeoff discussed in Section~\ref{ssec:succhalv_nvb} requires fewer resources per configuration.  Hence, if the envelope functions tighten quickly as a function of resource allocated, or the average distances between terminal losses is large, then \succhalv can be substantially faster than uniform allocation.
These intuitions are formalized in Section~\ref{sec:theory} and associated theorems/corollaries are provided that take into account the envelope functions and the distribution from which $\nu_i$'s are drawn. 

In practice, we do not have knowledge of either the envelope functions or the distribution of $\nu_i$'s, both of which are integral in characterizing \succhalv's required budget.  With \ouralg we address this shortcoming by hedging our aggressiveness.  We show in Section~\ref{sec:theory:guarantees_infinite} that \ouralg, despite having no knowledge of the envelope functions nor the distribution of $\nu_i$'s, requires a budget that is only log factors larger than that of \succhalv.


\section{Hyperparameter Optimization Experiments}\label{sec:experiments}
In this section, we evaluate the empirical behavior of \ouralg with three different resource types: iterations, data set subsamples, and feature samples.  For all experiments, we compare \ouralg with three well known Bayesian optimization algorithms---SMAC, TPE, and Spearmint---using their default settings.  We exclude Spearmint from the comparison set when there are conditional hyperparameters in the search space because it does not natively support them \citep{Eggensperger2013}.    
We also show results for \succhalv corresponding to repeating the most exploratory bracket of \ouralg to provide a baseline for aggressive early-stopping.\footnote{This is not done for the experiments in Section~\ref{sssec:117data}, since the most aggressive bracket varies from  dataset to dataset with the number of training points.} Additionally, as standard baselines against which to measure all speedups, we consider random search and ``random 2$\times$,'' a variant of
 random search with twice the budget of other methods.  
Of the hybrid methods described in Section~\ref{sec:related}, we compare to a variant of SMAC using the early termination criterion proposed by \citet{earlystopping2015} in the deep learning experiments described in Section~\ref{ssec:cnn}.  We think a comparison of \ouralg to more sophisticated hybrid methods introduced recently by \citet{fabolas2017} and \citet{kandasamy2017} is a fruitful direction for future work.

In the experiments below, we followed these loose guidelines when determining how to configuration \ouralg:
\begin{enumerate}
\item The maximum resource $\Rmax$ should be reasonable given the problem, but ideally large enough so that early-stopping is beneficial. 
\item $\eta$ should depend on $\Rmax$ and be selected to yield $\approx 5$ brackets with a minimum of 3 brackets.  This is to guarantee that \ouralg will use a baseline degree of early-stopping and prevent too coarse of a grid of $n$ vs $B$ tradeoffs.  
\end{enumerate}

\subsection{Early-Stopping Iterative Algorithms for Deep Learning}\label{ssec:cnn}
For this benchmark, we tuned a convolutional neural network\footnote{The model specification is available at \url{http://code.google.com/p/cuda-convnet/}.} 
with the same architecture as that used in \citet{Snoek2012} and \citet{earlystopping2015}.  The search spaces used in the two previous works differ, and we used a search space similar to that of \citet{Snoek2012} with 6 hyperparameters for stochastic gradient decent and 2 hyperparameters for the
response normalization layers (see Appendix \ref{add_results} for details).  In line with the two previous works, we used a batch size of 100 for all experiments.

\textbf{Data sets:}  We considered three image
classification data sets: CIFAR-10 \citep{cifar10data}, rotated MNIST with background images (MRBI) \citep{mrbidata}, and Street View House Numbers (SVHN) \citep{svhndata}. CIFAR-10 and SVHN contain $32\times 32$ RGB images while MRBI contains $28\times 28$ grayscale images.  Each data set was split into a training, validation, and test set: (1) CIFAR-10 has 40k, 10k, and 10k instances; (2) MRBI has 10k, 2k, and 50k instances; and (3) SVHN has close to 600k, 6k, and 26k instances for training, validation, and test respectively.  For all data sets, the only preprocessing performed on the raw images was demeaning. 

\textbf{\ouralg Configuration:} 
For these experiments, one unit of resource corresponds to 100 mini-batch iterations (10k examples with a batch size of 100).  
For CIFAR-10 and MRBI, $\Rmax$ was set to 300 (or 30k total iterations).
For SVHN, $\Rmax$ was set to 600 (or 60k total iterations) to accommodate the larger training set.  
Given $\Rmax$ for these experiments, we set $\eta=4$ to yield five \succhalv brackets for \ouralg.

\begin{figure}[h]
\centering
\subfigure[CIFAR-10]{\includegraphics[width=7cm,page=1,trim=20 0 20 20]{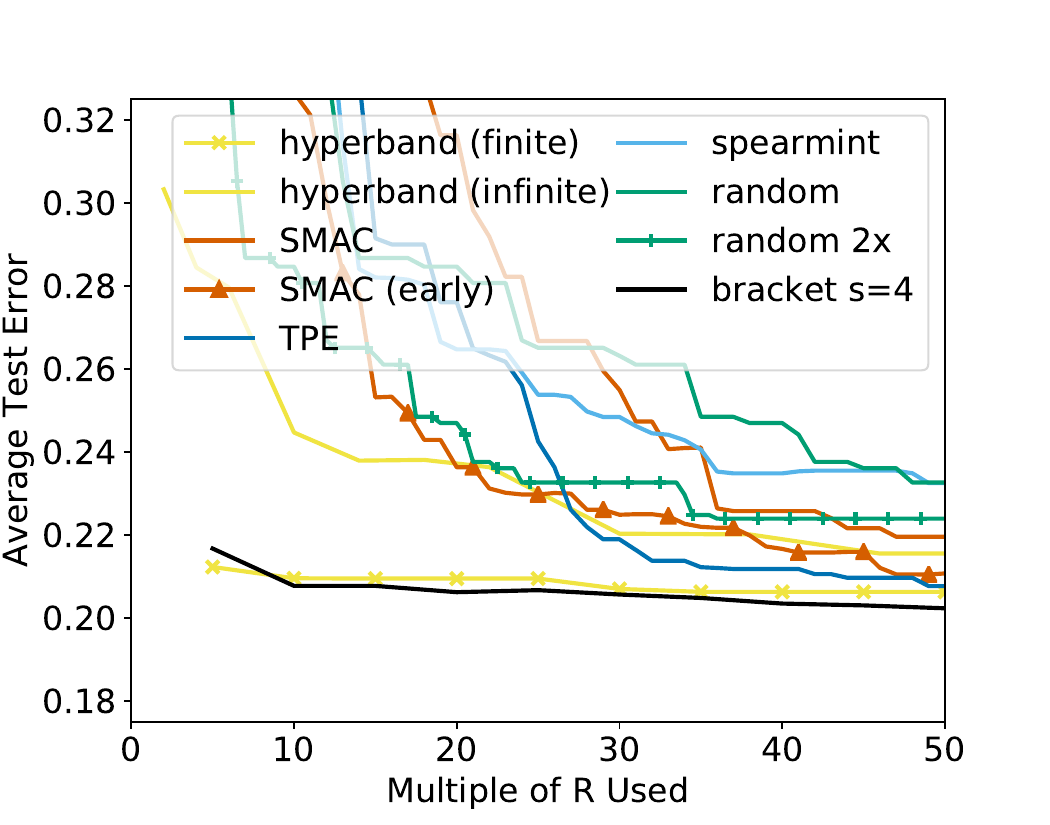}}
\subfigure[MRBI]{\includegraphics[width=7cm,page=2,trim=20 0 20 20]{error_avg_inf.pdf}}
\subfigure[SVHN]{\includegraphics[width=7cm,page=3,trim=20 0 20 20]{error_avg_inf.pdf}}
\caption{Average test error across 10 trials.  Label ``SMAC (early)'' corresponds to SMAC with the early-stopping criterion proposed in \citet{earlystopping2015} and label ``bracket $s=4$'' corresponds to repeating the most exploratory bracket of \ouralg.}
\label{fig:cnn_data}
\end{figure}

\textbf{Results:}  Each searcher was given a total budget of $50\Rmax$ per trial to return the best possible hyperparameter configuration. For \ouralg, the budget is sufficient to run the outer loop twice (for a total of 10 \succhalv brackets).  For SMAC, TPE, and random search, the budget corresponds to training 50 different configurations to completion.  Ten independent trials were performed for each searcher.  The experiments took the equivalent of over 1 year of GPU hours on NVIDIA GRID K520 cards available on Amazon EC2 \texttt{g2.8xlarge} instances. We set a total budget constraint in terms of iterations instead of compute time to make comparisons hardware independent.\footnote{Most trials were run on Amazon EC2 g2.8xlarge instances but a few trials were run on different machines due to the large computational demand of these experiments.}  Comparing progress by iterations instead of time ignores overhead costs, e.g. the cost of configuration selection for Bayesian methods and model initialization and validation costs for \ouralg.  
While overhead is hardware dependent, the overhead for \ouralg is below 5\% on EC2 \texttt{g2.8xlarge} machines, so comparing progress by time passed would not change results significantly.  

For CIFAR-10, the results in Figure \ref{fig:cnn_data}(a) show that \ouralg
is over an order-of-magnitude faster than its competitiors.  For MRBI, \ouralg is over an order-of-magnitude faster than standard configuration selection approaches and 5$\times$ faster than SMAC (early).  For SVHN, while \ouralg finds a good configuration faster, Bayesian optimization methods are competitive and SMAC (early) outperforms \ouralg.  The performance of SMAC (early) demonstrates there is merit to combining early-stopping and adaptive configuration selection.

Across the three data sets, \ouralg and SMAC (early) are the only two methods that consistently outperform random $2\times$.  On these data sets, \ouralg is over 20$\times$ faster than random search while SMAC (early) is $\leq7\times$ faster than random search within the evaluation window.
In fact, the first result returned by \ouralg after using a budget of 5$\Rmax$ is often competitive with results returned by other searchers
after using 50$\Rmax$. Additionally, \ouralg is less variable than
other searchers across trials, which is
highly desirable in practice (see Appendix \ref{add_results} for plots with error bars).

As discussed in Section~\ref{sec:eta}, for computationally expensive problems in high-dimensional search spaces, it may make sense to just repeat the most exploratory brackets.   Similarly, if meta-data is available about a problem or it is known that the quality of a configuration is evident after allocating a small amount of resource, then one should just repeat the most exploratory bracket.  Indeed, for these experiments, bracket $s=4$ vastly outperforms all other methods on CIFAR-10 and MRBI and is nearly tied with SMAC (early) for first on SVHN.  

While we set $\Rmax$ for these experiments to facilitate comparison to Bayesian methods and random search, it is also reasonable to use infinite horizon \ouralg to grow the maximum resource until a desired level of performance is reached.  We evaluate infinite horizon \ouralg on CIFAR-10 using $\eta=4$ and a starting budget of $B= 2\Rmax$.
Figure~\ref{fig:cnn_data}(a) shows that infinite horizon \ouralg is competitive with other methods but does not perform as well as finite horizon \ouralg within the 50$R$ budget limit.  The infinite horizon algorithm underperforms initially because it has to tune the maximum resource $R$ as well and starts with a less aggressive early-stopping rate.  This demonstrates that in scenarios where a max resource is known, it is better to use the finite horizon algorithm.  Hence, we focus on the finite horizon version of \ouralg for the remainder of our empirical studies.

Finally, CIFAR-10 is a very popular data set and state-of-the-art models achieve much lower error rates than what is shown in Figure~\ref{fig:cnn_data}.  The difference in performance is mainly attributable to higher model complexities and data manipulation (i.e.\ using reflection or random cropping to artificially increase the data set size).  If we limit the comparison to published results that use the same architecture and exclude data manipulation,  the best human expert result for the data set is 18\% error and the best hyperparameter optimized results are 15.0\% for \citet{Snoek2012}\footnote{We were unable to reproduce this result even after receiving the optimal hyperparameters from the authors through a personal communication.} and 17.2\% for  \citet{earlystopping2015}.  These results exceed ours on CIFAR-10 because they train on 25\% more data, by including the validation set, and also train for more epochs.  When we train the best model found by \ouralg on the combined training and validation data for 300 epochs, the model achieved a test error of 17.0\%.

\subsection{Data Set Subsampling}\label{ssec:subsampling}
We studied two different hyperparameter search optimization problems for which \ouralg uses data set subsamples as the resource.  The first adopts an extensive framework presented in \citet{Feurer2015} that attempts to automate preprocessing and model selection.  Due to certain limitations of the framework that fundamentally limited the impact of data set downsampling, we conducted a second experiment using a kernel classification task.
\subsubsection{117 Data Sets}\label{sssec:117data}

We used the framework introduced by \citet{Feurer2015}, which explored
a structured hyperparameter
 search space comprised of
 15 classifiers, 14 feature preprocessing methods, and 4 data
 preprocessing methods for a total of 110 hyperparameters.  We excluded the meta-learning component introduced in \citet{Feurer2015} used to warmstart Bayesian methods with promising configurations, in order to 
 perform a fair comparison with random search and \ouralg.  Similar to~\citet{Feurer2015}, we imposed a 3GB memory limit, a 6-minute timeout for each hyperparameter configuration and a one-hour time window to evaluate each searcher on each data set.  Twenty trials of each searcher were performed per data set and all trials in aggregate took over a year of CPU time on \texttt{n1-standard-1} instances from Google Cloud Compute.  Additional details about our experimental framework are available in Appendix~\ref{add_results}.

\textbf{Data sets}:
 \citet{Feurer2015} used 140 binary and multiclass classification data sets
from OpenML, but 23 of them are incompatible with the latest version of
the OpenML plugin~\citep{Feurer2015b}, 
so we worked with the remaining 117 data sets.
Due to the limitations of the experimental setup (discussed in Appendix~\ref{add_results}), we also separately considered 21 of these data sets, which demonstrated at least modest (though still
sublinear) training speedups due to subsampling.  Specifically, each of these
21 data sets showed on average at least a 3$\times$ speedup due to 8$\times$
downsampling on 100 randomly selected hyperparameter configurations. 

\textbf{\ouralg Configuration}:
Due to the wide range of dataset sizes, with some datasets having fewer than 10k training points, we ran \ouralg  with $\eta=3$ to allow for at least 3 brackets without being overly aggressive in downsampling on small datasets.  
$\Rmax$ was set to the full training set size for each data set and the maximum number of configurations for any bracket of \succhalv was limited to $n_{\max}=\max\{9,\Rmax/1000\}$.  This ensured that the most exploratory bracket of \ouralg will downsample at least twice.  As mentioned in Section~\ref{sec:eta}, when $n_{\max}$ is specified, the only difference when running the algorithm is $s_{\max}=\lfloor\log_\eta(n_{\max})\rfloor$ instead of $\lfloor\log_\eta(\Rmax)\rfloor$.

\begin{figure}[t]
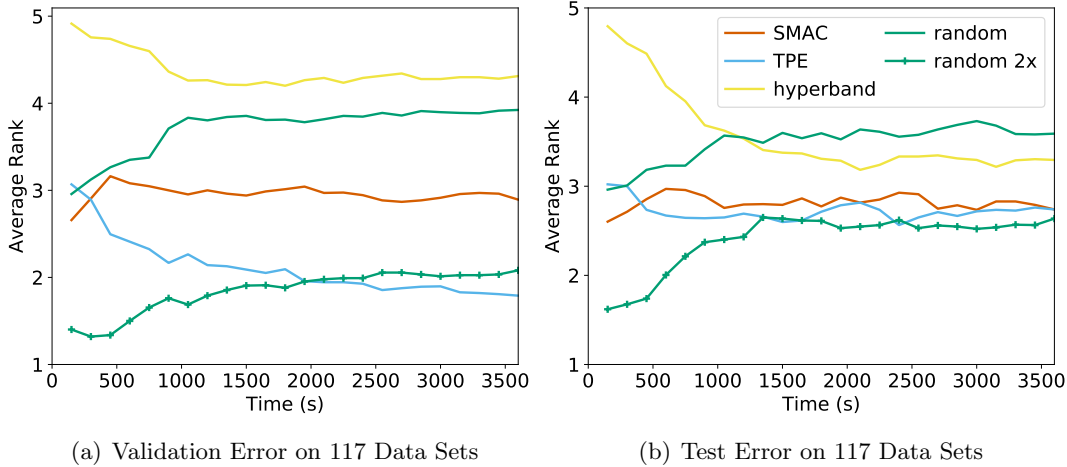
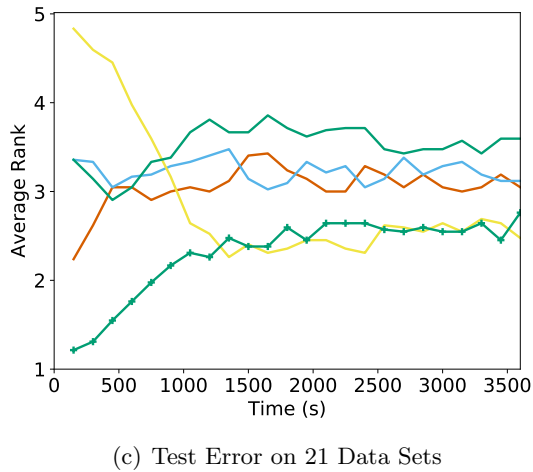

\centering
\subfigure[Validation Error on 117 Data Sets]
{\includegraphics[width=7cm,page=3,trim=10 0 0 0]{charts_117_datasets.pdf}}
\subfigure[Test Error on 117 Data Sets]
{\includegraphics[width=7cm,page=4,trim=10 0 0 0]{charts_117_datasets.pdf}}
\subfigure[Test Error on 21 Data Sets]
{\includegraphics[width=7cm,page=6,trim=10 0 0 0]{charts_117_datasets.pdf}}
\caption{Average rank across all data sets for each searcher.  For each data set,
the searchers are ranked according to the average validation/test error across
20 trials.} \label{fig:all_data sets}
\end{figure}
 
\textbf{Results:}   The results on all 117 data sets in
Figure~\ref{fig:all_data sets}(a,b) show that \ouralg outperforms random search
in test error rank despite performing worse in validation error rank. Bayesian methods outperform \ouralg and random search
in test error performance but also exhibit signs of overfitting to the validation set, as they outperform \ouralg by a larger margin 
on the validation error rank.  Notably, random 2$\times$ outperforms all other
methods.  
However, for the subset of 21 data sets,
Figure \ref{fig:all_data sets}(c) shows that \ouralg outperforms all other searchers on
test error rank, including random 2$\times$ by a very small margin.  While these results are more promising,
the effectiveness of \ouralg was restricted
in this experimental framework; for smaller data sets, the startup
overhead was high relative to total training time, while for larger data sets,
only a handful of configurations could be trained within the hour window.  

\begin{figure}[h]
\centering
\includegraphics[width=8cm,page=1,trim=10 0 10 0]{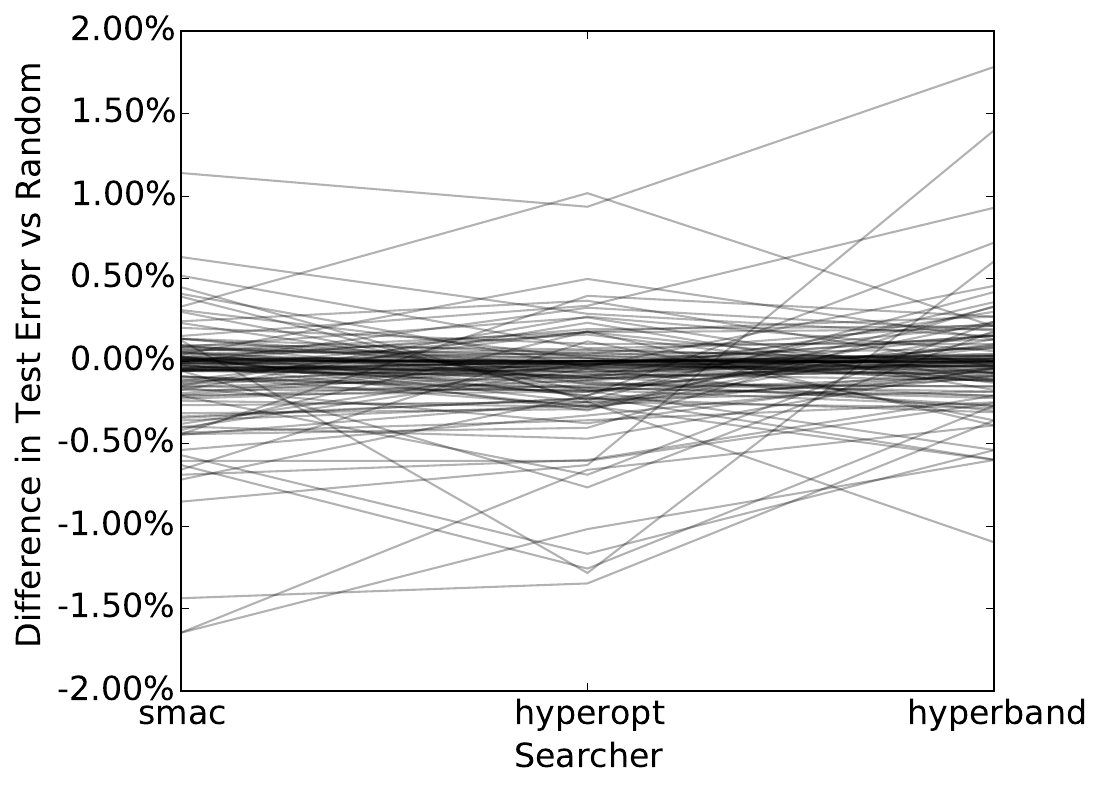}
\caption{Each line plots, for a single data set, the difference in test error versus random search for each search, where lower is better.  Nearly all the lines fall within the -0.5\% and 0.5\% band and, with the exception of a few outliers, the lines are mostly flat.  } \label{fig:all_data sets_dist}
\end{figure}

We note that while average rank plots like those in Figure~\ref{fig:all_data sets} are an effective way to aggregate information across many searchers and data sets, they provide no indication about the {\em magnitude} of the differences between the performance of the methods. Figure \ref{fig:all_data sets_dist}, which charts the difference between the test error for each searcher and that of random search across all 117 datasets, 
 highlights the small difference in the magnitude of the test errors across searchers.

These results are not surprising;
as mentioned in Section~\ref{ssec:related_hyper}, vanilla Bayesian optimization methods perform similarly to random search in high-dimensional search spaces.  \citet{Feurer2015} showed that using meta-learning to warmstart Bayesian optimization methods 
improved performance in this high-dimensional setting.  Using meta-learning to identify a promising
distribution from which to sample configurations as input into \ouralg is a direction for future work.

\subsubsection{Kernel Regularized Least Squares Classification}\label{ssec:kernel_lsqr}
For this benchmark, we tuned the hyperparameters of a kernel-based classifier on CIFAR-10.
We used the multi-class regularized least squares classification model, which is known to have comparable performance to SVMs \citep{rifkin2004defense,agarwal2014least} but can be trained significantly faster.\footnote{The default SVM method in Scikit-learn is single core and takes hours to train on CIFAR-10, whereas a block coordinate descent least squares solver takes less than 10 minutes on an 8 core machine.}  The hyperparameters considered in the search space include preprocessing method, regularization, kernel type, kernel length scale, and other kernel specific hyperparameters (see Appendix~\ref{add_results} for more details).  For \ouralg, we set $\Rmax=400$, with each unit of resource representing 100 datapoints, and $\eta=4$ to yield a total of 5 brackets.   Each hyperparameter optimization algorithm was run for ten trials on Amazon EC2 \texttt{m4.2xlarge}  instances; for a given trial, \ouralg was allowed to run for two outer loops, bracket $s=4$ was repeated 10 times, and all other searchers were run for 12 hours.
\begin{figure}
\centering
\begin{minipage}[b]{.47\textwidth}
  \centering
  \includegraphics[width=7cm,page=1,trim=10 10 10 10]{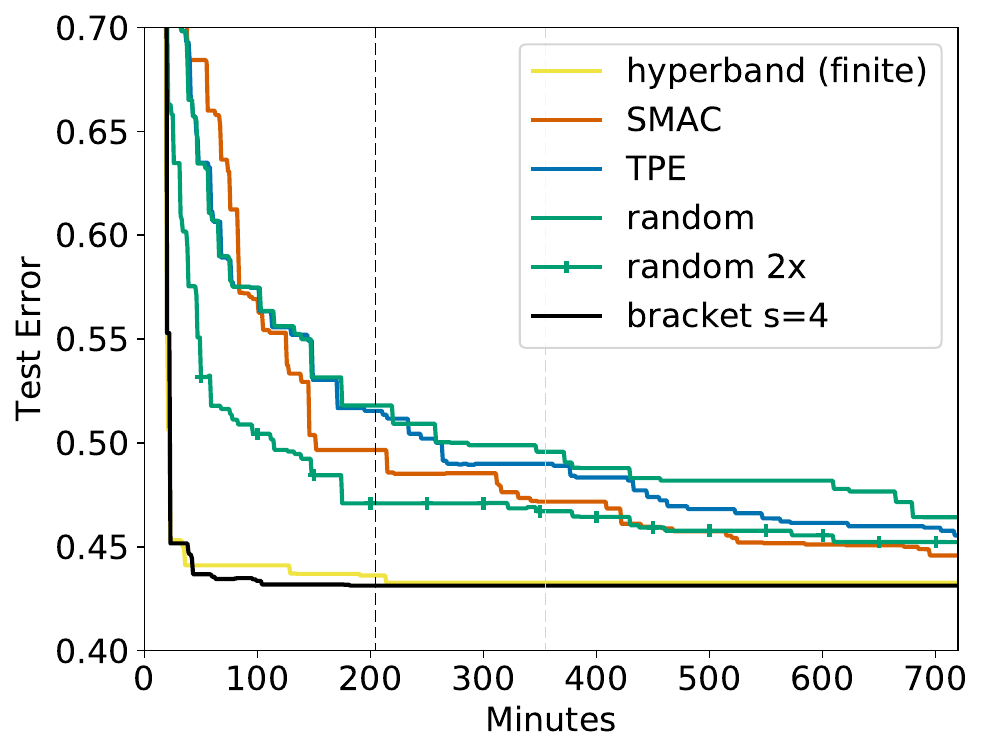}
\captionof{figure}{Average test error of the best kernel regularized least square classification model found by each searcher on CIFAR-10.  The color coded dashed lines indicate when the last trial of a given searcher finished.  } \label{fig:kernel_lsqr}
\end{minipage}%
\hspace{0.5cm}
\begin{minipage}[b]{.47\textwidth}
  \centering
\includegraphics[width=7cm,page=1,trim=10 10 10 10]{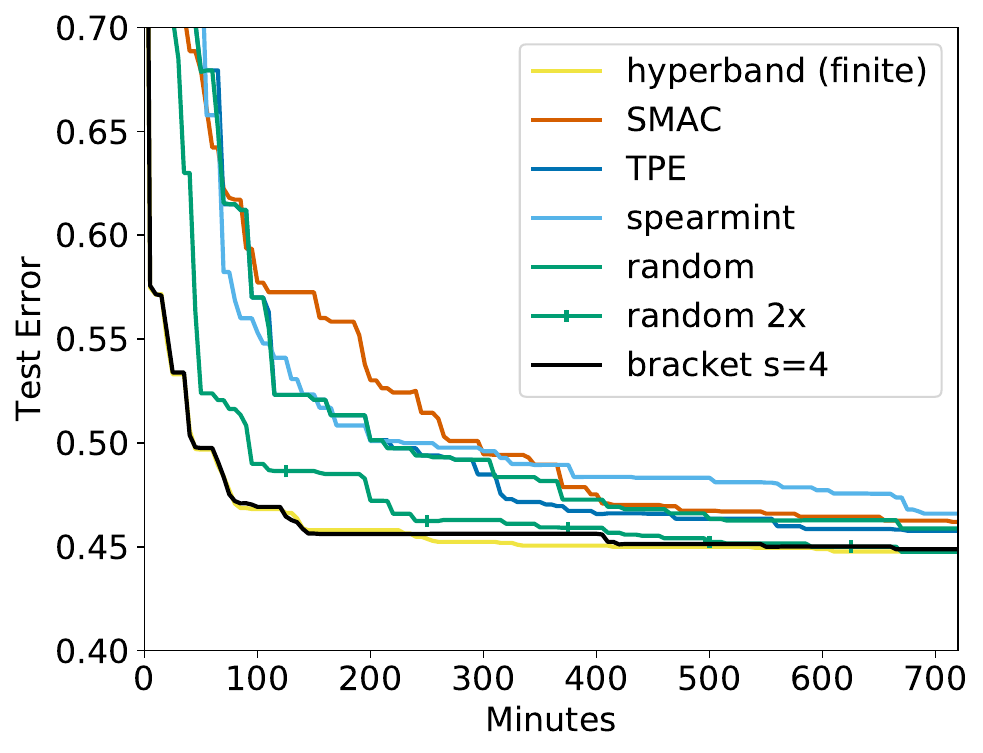}
\captionof{figure}{Average test error of the best random features model found by each searcher on CIFAR-10.  The test error for \ouralg and bracket $s=4$ are calculated in every evaluation instead of at the end of a bracket.} \label{fig:kernel_approx}
\end{minipage}
\end{figure}

Figure~\ref{fig:kernel_lsqr} shows that \ouralg returned a good configuration after completing the first \succhalv bracket in approximately 20 minutes; other searchers failed to reach this error rate on average even after the entire 12 hours. Notably, \ouralg was able to evaluate over 250 configurations in this first bracket of \succhalv, while competitors were able to evaluate only three configurations in the same amount of time.   Consequently, \ouralg is over 30$\times$ faster than Bayesian optimization methods and 70$\times$ faster than random search.   Bracket $s=4$ sightly outperforms \ouralg but the terminal performance for the two algorithms are the same.  Random 2$\times$ is competitive with SMAC and TPE. 

\subsection{Feature Subsampling to Speed Up Approximate Kernel Classification}\label{ssec:random_features}

Next, we examine the performance of \ouralg when using features as a resource on a random feature kernel approximations task.  Features were randomly generated using the method described in \citet{randomfeatures2007} to approximate the RBF kernel, and these random features were then used as inputs to a ridge regression classifier. The hyperparameter search space included the preprocessing method, kernel length scale, and $L_2$ penalty.  While it may seem natural to use infinite horizon \ouralg, since the fidelity of the approximation improves with more random features, in practice, the amount of available machine memory imposes a natural  upper bound on the number of features.  Thus, we used finite horizion \ouralg with a maximum resource of 100k random features, which  comfortably fit into a machine with 60GB of memory.  Additionally, we set one unit of resource to be 100 features, so $\Rmax=1000$.  Again, we set  $\eta=4$ to yield 5 brackets of \succhalv.  We ran 10 trials of each searcher, with each trial lasting 12 hours on a \texttt{n1-standard-16} machine from Google Cloud Compute.  The results in Figure~\ref{fig:kernel_approx} show that \ouralg is around 6$\times$ faster than Bayesian methods and random search.  \ouralg performs similarly to bracket $s=4$.  Random 2$\times$ outperforms Bayesian optimization algorithms. 

\subsection{Experimental Discussion}
While our experimental results show \ouralg is a promising algorithm for hyperparameter optimization, a few questions 
naturally arise:
\begin{enumerate}
\item What impacts the speedups provided by \ouralg?
\item Why does \succhalv seem to outperform \ouralg?
\item What about hyperparameters that should depend on the resource?
\end{enumerate}
We next address each of these questions in turn.

\subsubsection{Factors Impacting the Performance of \ouralg}
For a given $\Rmax$, the most exploratory \succhalv round performed by \ouralg evaluates $\Rmax$ configurations using a budget of $(\lfloor\log_\eta(\Rmax)\rfloor+1)\Rmax$, which gives an upper bound on the potential speedup over random search.  If training time scales linearly with the resource, the maximum speedup offered by \ouralg compared to random search is $\frac{\Rmax}{(\lfloor\log_\eta(\Rmax)\rfloor+1)}$.  For the values of $\eta$ and $\Rmax$ used in our experiments, the maximum speedup over random search is approximately $50\times$ given linear training time.   However, we observe a range of speedups from $6\times$ to $70\times$ faster than random search.  The differences in realized speedup can be explained by three factors:
\begin{enumerate} 
\item \emph{How training time scales with the given resource}. In cases where training time is superlinear as a function of the resource,  \ouralg can offer higher speedups.  For instance, if training scales like a polynomial of degree $p>1$, the maximum speedup for \ouralg over random search is approximately $\frac{\eta^{p}-1}{\eta^{p-1}}\Rmax$.  In the kernel least square classifier experiment discussed in Section~\ref{ssec:kernel_lsqr}, the training time scaled quadratically as a function of the resource, 
 which explains why the realized speedup of $70\times$ is higher than the maximum expected speedup given linear scaling. 
 \item \emph{Overhead costs associated with training}. Total evaluation time also depends on fixed overhead costs associated with evaluating each hyperparameter configuration, e.g., initializing a model, resuming previously trained models, and  calculating validation error. 
 For example, in the downsampling experiments on 117 data sets presented in Section~\ref{sssec:117data}, \ouralg did not provide significant speedup because many data sets could be trained in a matter of a few seconds and the initialization cost was high relative to training time.
\item \emph{The difficulty of finding a good configuration}. Hyperparameter optimization problems can vary in difficulty.  For instance, an `easy' problem is one where a randomly sampled configuration is likely to result in a high-quality model, and thus we only need to evaluate a small number of configurations to find a good setting.  In contrast, a `hard' problem is one where an arbitrary configuration is likely to be bad, in which case many configurations must be considered.  \ouralg leverages downsampling to boost the number of configurations that are evaluated, and thus is better suited for `hard' problems where
more evaluations are actually necessary to find a good setting.  
 Generally, the difficulty of a problem scales with the dimensionality of the search space. For low-dimensional problems, the number of configurations evaluated by random search and Bayesian methods is exponential in the number of dimensions so good coverage can be achieved.  For instance, the low-dimensional ($d=3$) search space in our feature subsampling experiment in Section~\ref{ssec:random_features} helps explain why 
\ouralg is only $6\times$ faster than random search.  
In contrast, for the neural network experiments in Section~\ref{ssec:cnn}, we hypothesize that faster speedups are observed for \ouralg because the dimension of the search space is higher.
 \end{enumerate}

\subsubsection{Comparison to \succhalv}
With the exception of the LeNet experiment (Section~\ref{ssec:lenet}) and the 117 Datasets experiment (Section~\ref{sssec:117data}), the most aggressive bracket of \succhalv outperformed \ouralg in all of our experiments.  In hindsight, we should have just run bracket $s=4$, since aggressive early-stopping 
provides massive speedups on many of these benchmarking tasks.  However, as previously mentioned, it was unknown a priori that bracket $s=4$ would perform the best and that is why we have to cycle through all possible brackets with \ouralg.  Another question is what happens when one increases $s$ even further, i.e.\ instead of 4 rounds of elimination, why not 5 or even more with the same maximum resource $R$?  In our case, $s=4$ was the most aggressive bracket we could run given the minimum resource per configuration limits imposed for the previous experiments. However, for larger data sets, it is possible to extend the range of possible values for $s$, in which case, \ouralg may either provide even faster speedups if more aggressive early-stopping helps or be slower by a small factor if the most aggressive brackets are essentially throwaways.  

We believe prior knowledge about a task can be particularly useful for limiting the range of brackets explored by \ouralg.  In our experience, aggressive early-stopping is generally safe for neural network tasks and even more aggressive early-stopping may be reasonable for larger data sets and longer training horizons.  However, when pushing the degree of early-stopping by increasing $s$, one has to consider the additional overhead cost associated with examining more models.  Hence, one way to leverage meta-learning would be to use  learning curve convergence rate, difficulty of different search spaces, and overhead costs of related tasks to determine the brackets considered by \ouralg. 
\subsubsection{Resource Dependent Hyperparameters}
In certain cases, the setting for a given hyperparameter should depend on the allocated resource.  For example, the maximum tree depth regularization hyperparameter for random forests should be higher with more data and more features.  However, the optimal tradeoff between maximum tree depth and the resource is unknown and can be data set specific.  In these situations, the rate of convergence to the true loss is usually slow because the performance on a smaller resource is not indicative of that on a larger resource.  Hence, these problems are particularly difficult for \ouralg, since the benefit of early-stopping can be muted.  Again, while \ouralg will only be a small factor slower than that of \succhalv with the optimal early-stopping rate, 
we recommend removing the dependence of the hyperparameter on the resource if possible.   For the random forest example, an alternative regularization hyperparameter is minimum samples per leaf, which is less dependent on the training set size.  Additionally, the dependence can oftentimes be removed with simple normalization.
For example, the regularization term for our kernel least squares experiments were normalized by the training set size to maintain a constant tradeoff between the mean-squared error and the regularization term.


\section{Theory}
\label{sec:theory}

In this section, we introduce the pure-exploration non-stochastic infinite-armed bandit (NIAB) problem, a very general setting which encompasses our hyperparameter optimization problem of interest. 
As we will show, \ouralg is in fact applicable to problems far beyond just hyperparameter optimization. 
We begin by formalizing the hyperparameter optimization problem and then reducing it to the pure-exploration NIAB problem.  We subsequently present a detailed analysis of \ouralg in both the infinite and finite horizon settings.

\subsection{Hyperparameter Optimization Problem Statement}

Let $\X$ denote the space of valid hyperparameter configurations, which could include continuous, discrete, or categorical variables that can be constrained with respect to each other in arbitrary ways (i.e.\ $\X$ need not be limited to a subset of $[0,1]^d$). 
For $k=1,2,\dots$ let $\ell_k \colon \X \rightarrow [0,1]$ be a sequence of loss functions defined over $\X$. 
For any hyperparameter configuration $x \in \X$, $\ell_k(x)$ represents the validation error of the model trained using $x$ with $k$ units of resources (e.g. iterations).
In addition, for some $\Rmax \in \mathbb{N} \cup \{ \infty \}$, define $\ell_* = \lim_{k \rightarrow R} \ell_k$
and $\nu_* =  \inf_{x \in \X} \ell_*(x)$. 
Note that $\ell_k(\cdot)$ for all $k\in \mathbb{N}$, $\ell_*(\cdot)$, and $\nu_*$ are all unknown to the algorithm a priori.
In particular, it is uncertain how quickly $\ell_k(x)$ varies as a function of $x$ for any fixed $k$, and how quickly $\ell_k(x)\rightarrow \ell_*(x)$ as a function of $k$ for any fixed $x \in \X$.

We assume hyperparameter configurations are sampled randomly from a known probability distribution 
$p(x)\colon\X\rightarrow [0,\infty)$, with support on $\X$.  In our experiments, $p(x)$ is simply the uniform distribution, but the algorithm can be used with any sampling method.
If $X \in \X$ is a random sample from this probability distribution, then $\ell_*(X)$ is a random variable whose distribution is unknown since $\ell_*(\cdot)$ is unknown.
Additionally, since it is unknown how $\ell_k(x)$ varies as a function of $x$ or $k$, one cannot necessarily infer anything about $\ell_k(x)$ given knowledge of $\ell_j(y)$ for any $j \in \mathbb{N}$, $y \in \X$. 
As a consequence, we reduce the hyperparmeter optimization problem down to a much simpler problem that ignores all underlying structure of the hyperparameters: we only interact with some $x \in \X$ through its loss sequence $\ell_k(x)$ for $k=1,2,\dots$. 
With this reduction, the particular value of $x \in \X$ does nothing more than index or uniquely identify the loss sequence.

Without knowledge of how fast $\ell_k(\cdot) \rightarrow \ell_*(\cdot)$ or how $\ell_*(X)$ is distributed, the goal of \ouralg is to identify a hyperparameter configuration $x \in \X$ that minimizes $\ell_*(x) - \nu_*$ by drawing as many random configurations as desired while using as few total resources as possible. 

\subsection{The Pure-Exploration Non-stochastic Infinite-Armed Bandit Problem}\label{sec:theory_assumptions}

We now formally define the bandit problem of interest, and relate it to the problem of hyperparameter optimization.
Each ``arm'' in the NIAB game is associated with a sequence that is drawn randomly from a distribution over sequences. 
If we ``pull'' the $i$th drawn arm exactly $k$ times, we observe a loss $\ell_{i,k}$. 
At each time, the player can either draw a new arm (sequence) or pull a previously drawn arm an additional time. 
There is no limit on the number of arms that can be drawn.
We assume the arms are identifiable only by their index $i$ (i.e.\ we have no side-knowledge or feature representation of an arm), and we also make the following two additional assumptions:
\begin{assumption}\label{eqn:NIAB_assumption_limit}
For each $i \in \mathbb{N}$ the limit $\lim_{k \rightarrow \infty} \ell_{i,k}$ exists and is equal to $\nu_i$.\footnote{We can always define $\ell_{i,k}$ so that convergence is guaranteed, i.e.\ taking the infimum of a sequence.}
\end{assumption}
\begin{assumption}\label{eqn:NIAB_assumption_F}
Each $\nu_i$ is a bounded  i.i.d. random variable with cumulative distribution function~$F$.
\end{assumption}
The objective of the NIAB problem is to identify an arm $\hat{\imath}$ with small $\nu_{\hat{\imath}}$ using as few total pulls as possible.
We are interested in characterizing $\nu_{\hat{\imath}}$ as a function of the total number of pulls from all the arms.
Clearly, the  hyperparameter optimization problem described above is an instance of the  NIAB problem. In this case, arm $i$ correspondes to a configuration $x_i\in \X$, with $\ell_{i,k} = \ell_k(x_i)$; Assumption~\ref{eqn:NIAB_assumption_limit} is equivalent to requiring that $\nu_i = \ell_*(x_i)$ exists; and Assumption~\ref{eqn:NIAB_assumption_F} follows from the fact that the arms are drawn i.i.d. from $\X$ according to distribution function $p(x)$.  $F$ is simply the cumulative distribution function of $\ell_*(X)$, where $X$ is a random variable drawn from the distribution $p(x)$ over $\X$.  Note that since the arm draws are independent, the $\nu_i$'s are also independent.  Again, this is not to say that the validation losses do not depend on the settings of the hyperparameters; the validation loss could well be correlated with certain hyperparameters, but this is not used in the algorithm and no assumptions are made regarding the correlation structure.  

In order to analyze the behavior of \ouralg in the NIAB setting, we must define a few additional objects.  Let $\nu_*=\inf\{ m : \P\left(\nu \leq m \right) > 0 \} > -\infty$, since the domain of the distribution $F$ is bounded.
Hence, the cumulative distribution function $F$ satisfies
\begin{align} \label{dist_assumption}
\P\left( \nu_i - \nu_* \leq \epsilon \right) = F(\nu_*+\epsilon)
\end{align}
and let $F^{-1}(y) = \inf_x \{ x : F(x)\leq y \}$. 
Define $\gamma \colon \mathbb{N} \rightarrow \R$ as the pointwise smallest, monotonically decreasing function satisfying
\begin{align}
\sup_{i} | \ell_{i,j} - \ell_{i,*} | \leq \gamma(j) \ ,  \ \ \forall j \in \mathbb{N}.
\end{align}
The function $\gamma$ is guaranteed to exist by Assumption~\ref{eqn:NIAB_assumption_limit} and bounds the deviation from the limit value as the sequence of iterates $j$ increases. 
For hyperparameter optimization, this follows from the fact that $\ell_k$ uniformly converges to $\ell_*$ for all $x\in\X$.  In addition, $\gamma$ can be interpretted as the deviation of the validation error of a configuration trained on a subset of resources versus the maximum number of allocatable resources.  
Finally, define $\Rmax$ as the first index such that $\gamma(\Rmax)=0$ if it exists, otherwise set $\Rmax=\infty$.  
For $y \geq 0$ let $\gamma^{-1}(y) = \min \{j \in \mathbb{N} \colon \gamma(j) \leq y \}$, using the convention that $\gamma^{-1}(0):= R$ which we recall can be infinite.

As previously discussed, there are many real-world scenarios in which $\Rmax$ is finite and known. 
For instance, if increasing subsets of the full data set is used as a resource, then the maximum number of resources cannot exceed the full data set size, and thus $\gamma(k)=0$ for all $k \geq R$ where $\Rmax$ is the (known) full size of the data set.
In other cases such as iterative training problems, one might not want to or know how to bound $\Rmax$. We separate these two settings into the \emph{finite horizon} setting where $\Rmax$ is finite and known, and the \emph{infinite horizon} setting where no bound on $\Rmax$ is known and it is assumed to be infinite.  While our empirical results suggest that the finite horizon may be more practically relevant for the problem of hyperparameter optimization, 
the infinite horizon case has natural connections to the literature,  and we begin by analyzing this setting.

\subsection{Infinite Horizon Setting ($\Rmax=\infty$)}
\label{ssec:theory_infinite}

\begin{figure}[t]
\centerline{
\fbox{\parbox[b]{.75\textwidth}{{\underline{\succhalv} (Infinite horizon) } \\[2pt] \small
{\bf Input}: Budget $B$, $n$ arms where $\ell_{i,k}$ denotes the $k$th loss
from the $i$th arm \\[3pt]
{\bf Initialize}: $S_0 = [n]$. \\[3pt] 
\textbf{For} $k=0,1,\dots,\lceil \log_2(n) \rceil -1$ \\ [3pt]
\indent \hspace{.4cm} Pull each arm in $S_k$ for $r_k = \lfloor \frac{ B
}{|S_k| \lceil \log_2(n) \rceil } \rfloor$ times.\\[3pt]
\indent \hspace{.4cm} Keep the best $ \lfloor |S_k| /2 \rfloor$ arms in terms of the $r_k$th observed loss as $S_{k+1}$.\\[3pt]
{\bf Output} : $\hat{\imath}$, $\ell_{\hat{\imath},\lfloor \tfrac{ B/2 }{ \lceil \log_2(n) \rceil} \rfloor}$ where $\hat{\imath} = S_{\lceil \log_2(n) \rceil}$
}}
}
\centering
\fbox{\parbox[b]{.75\textwidth}{{\underline{\ouralg} (Infinite horizon)}  \\[2pt]  \small
\textbf{Input:} None \\[3pt]
\textbf{For} $k=1,2,\dots$\\[3pt]
\indent \hspace{.4cm} \textbf{For} $s\in\mathbb{N}$ s.t. $k-s\geq\log_2(s)$ \\[3pt]
\indent \hspace{.8cm} $B_{k,s}=2^k$, $n_{k,s} = 2^{s}$ \\[3pt]
\indent \hspace{.8cm} $\hat{\imath}_{k,s}, \ell_{\hat{\imath}_{k,s}, \lfloor \tfrac{ 2^{k-1} }{ s } \rfloor} \leftarrow$ \succhalv($B_{k,s}$,$n_{k,s}$) 
}}
\caption{(Top) The \succhalv algorithm proposed and analyzed in \cite{JamiesonTalwalkar2015} for the non-stochastic setting. Note this algorithm was originally proposed for the stochastic setting in \cite{karnin2013almost}.  (Bottom) The \ouralg algorithm for the infinite horizon setting. \ouralg calls \succhalv as a subroutine.}
\label{alg:hyperband_infinite}
\end{figure} 

Consider the \ouralg algorithm of Figure~\ref{alg:hyperband_infinite}. The algorithm uses \succhalv (Figure~\ref{alg:hyperband_infinite}) as a subroutine that takes a finite set of arms as input and outputs an estimate of the best performing arm in the set. 
We first analyze \succhalv  (SH) for a given set of limits $\nu_i$ and then consider the performance of SH when $\nu_i$ are drawn randomly according to $F$. We then analyze the \ouralg algorithm. We note that the algorithm of Figure~\ref{alg:hyperband_infinite} was originally proposed by \cite{karnin2013almost} for the stochastic setting. However,  \citet{JamiesonTalwalkar2015} analyzed it in the non-stochastic setting and also found it to work well in practice. Extending the result of \cite{JamiesonTalwalkar2015} we have the following theorem:

\begin{theorem}
\label{thm:succ_halving}
Fix $n$ arms. Let $\nu_i = \displaystyle\lim_{\tau \rightarrow \infty} \ell_{i,\tau}$ and assume $\nu_1 \leq \dots \leq \nu_n$. For any $\epsilon >0$ let 
\begin{align*}
z_{SH} &=  2 \lceil \log_2(n) \rceil   \, \max_{i=2,\dots,n}  i \, ( 1 + \gamma^{-1} \left(    \max\left\{ \tfrac{\epsilon}{4} , \tfrac{\nu_{i}- \nu_{1}}{2} \right\} \right) )\\  &
\leq 2 \lceil \log_2(n) \rceil  \big( n +  \sum_{i=1,\dots,n}  \gamma^{-1} \left(   \max\left\{ \tfrac{\epsilon}{4} , \tfrac{\nu_{i}- \nu_{1}}{2} \right\} \right) \big)
\end{align*} 
If the \succhalv algorithm of Figure~\ref{alg:hyperband_infinite} is run with any budget $B > z_{SH}$ then an arm $\hat{\imath}$ is returned that satisfies $\nu_{\hat{\imath}} - \nu_1 \leq \epsilon/2$. Moreover, $|\ell_{\hat{\imath},\lfloor \tfrac{ B/2 }{ \lceil \log_2(n) \rceil} \rfloor} - \nu_1| \leq \epsilon$.
\end{theorem}

The next technical lemma will be used to characterize the problem dependent term $ \sum_{i=1,\dots,n}  \gamma^{-1} \left(   \max\left\{ \tfrac{\epsilon}{4} , \tfrac{\nu_{i}- \nu_{1}}{2} \right\} \right)$ when the sequences are drawn from a probability distribution. 

\begin{lemma} \label{lem:gamma_sum_bound_rand_arms}
Fix $\delta \in (0,1)$. Let $p_n = \frac{\log(2/\delta)}{n}$. For any $\epsilon \geq 4(F^{-1}(p_n) - \nu_*)$ define 
\begin{align*}
\mathbf{H}(F,\gamma,n,\delta,\epsilon) := 2n  \int_{ \nu_*+\epsilon/4 }^\infty \gamma^{-1}(\tfrac{t - \nu_{*}}{4}) dF(t)  + \left(\tfrac{4}{3} \log(2/\delta) + 2n F(\nu_*+\epsilon/4) \right) \gamma^{-1}\left( \tfrac{\epsilon}{16} \right) 
\end{align*}
and $\mathbf{H}(F,\gamma,n,\delta) := \mathbf{H}(F,\gamma,n,\delta,4(F^{-1}(p_n) - \nu_*))$ so that 
\begin{align*}
\mathbf{H}(F,\gamma,n,\delta)  = 2n  \int_{ p_n }^1 \gamma^{-1}(\tfrac{F^{-1}(t) - \nu_{*}}{4}) dt  + \tfrac{10}{3} \log(2/\delta)  \gamma^{-1}\left(\tfrac{ F^{-1}(p_n) - \nu_*}{4}\right) .
\end{align*}
For $n$ arms with limits $\nu_1 \leq \dots \leq \nu_n$ drawn from $F$, then 
\begin{align*}
\nu_1 \leq F^{-1}(p_n) \quad \text{ and } \quad \sum_{i=1}^n &\gamma^{-1} \left( \max\left\{ \tfrac{\epsilon}{4} , \tfrac{\nu_{i}- \nu_{1}}{2} \right\} \right)  \leq \mathbf{H}(F,\gamma,n,\delta,\epsilon)
\end{align*}
for any $\epsilon \geq 4(F^{-1}(p_n) - \nu_*)$ with probability at least $1-\delta$.
\end{lemma}

Setting $\epsilon = 4(F^{-1}(p_n)-\nu_*)$ in Theorem~\ref{thm:succ_halving} and using the result of Lemma~\ref{lem:gamma_sum_bound_rand_arms} that $\nu_* \leq \nu_1 \leq \nu_* + (F^{-1}(p_n)-\nu_*)$, we immediately obtain the following corollary.
\begin{corollary}
\label{cor:succ_halv_rand_arms_infinite}
Fix $\delta \in (0,1)$ and $\epsilon \geq 4(F^{-1}(\tfrac{\log(2/\delta)}{n}) - \nu_*)$. Let $B = 4 \lceil \log_2(n) \rceil  \mathbf{H}(F,\gamma,n,\delta,\epsilon)$
where $\mathbf{H}(F,\gamma,n,\delta,\epsilon)$ is defined in Lemma~\ref{lem:gamma_sum_bound_rand_arms}. 
If the \succhalv algorithm of Figure~\ref{alg:hyperband_infinite} is run with the
specified $B$ and $n$ arm configurations drawn randomly according to $F$, then
an arm $\hat{\imath} \in [n]$ is returned such that with probability at least
$1-\delta$ we have $\nu_{\hat{\imath}} - \nu_* \leq \big(F^{-1}(\frac{\log(2/\delta)}{n}) - \nu_*\big) + \epsilon/2$.
In particular, if $B = 4 \lceil \log_2(n) \rceil  \mathbf{H}(F,\gamma,n,\delta)$ and $\epsilon = 4(F^{-1}(\tfrac{\log(2/\delta)}{n}) - \nu_*)$ then $\nu_{\hat{\imath}} - \nu_* \leq 3\big(F^{-1}(\frac{\log(2/\delta)}{n}) - \nu_*\big)$ with probability at least $1-\delta$.
\end{corollary}

Note that for any fixed $n \in \mathbb{N}$ we have for any $\Delta >0$
\begin{align*}
\P( \min_{i=1,\dots,n} \nu_i -\nu_* \geq \Delta  ) = (1-F(\nu_*+\Delta))^n \approx e^{-n F(\nu_*+\Delta)}
\end{align*}
which implies $\E[ \min_{i=1,\dots,n} \nu_i - \nu_*] \approx F^{-1}(\tfrac{1}{n})-\nu_*$. 
That is, $n$ needs to be sufficiently large so that it is probable that a good limit is sampled. 
On the other hand, for any fixed $n$, Corollary~\ref{cor:succ_halv_rand_arms_infinite} suggests that the total resource budget $B$ needs to be large enough in order to overcome the rates of convergence of the sequences described by $\gamma$. 
Next, we relate  SH to a naive approach that uniformly allocates resources to a fixed set of $n$ arms.

\subsubsection{Non-Adaptive Uniform Allocation}
The non-adaptive uniform allocation strategy takes as inputs a budget $B$ and $n$ arms,  allocates $B/n$ to each of the arms, and picks the arm with the lowest loss. The following results  allow us to compare with \succhalv.

\begin{proposition} \label{prop:uniform_sampling_bound}
Suppose we draw $n$ random configurations from $F$, train each with
$j=\min\{B/n,R\}$ iterations, and let $\hat{\imath} = \arg\min_{i
=1,\dots,n} \ell_{j}(X_i)$. Without loss
of generality assume $\nu_1 \leq \ldots \leq \nu_n$. If
\begin{align}
B \geq n \gamma^{-1}\left( \tfrac{1}{2} (F^{-1}(\tfrac{\log(1/\delta)}{n})  - \nu_*) \right) 
\label{eq:unif_budget}
\end{align}
then with probability at least $1-\delta$ we have $\nu_{\hat{\imath}} - \nu_* \leq 2 \left( F^{-1}\left( \tfrac{ \log(1/\delta)}{n} \right) - \nu_* \right)$. In contrast, there exists a sequence of functions $\ell_j$ that satisfy $F$ and $\gamma$ such that if 
\begin{align*}
B \leq  n \gamma^{-1}\left( 2 (F^{-1}(\tfrac{\log(c/\delta)}{n+\log(c/\delta)})  - \nu_*) \right)
\end{align*}
then with probability at least $\delta$,   we have $\nu_{\hat{\imath}} - \nu_* \geq 2(F^{-1}(\tfrac{\log(c/\delta)}{n+\log(c/\delta)}) - \nu_*)$, where $c$ is a constant that depends on the regularity of $F$.
\end{proposition}

For any fixed $n$ and sufficiently large $B$, Corollary~\ref{cor:succ_halv_rand_arms_infinite} shows that \succhalv outputs an $\hat{\imath} \in [n]$ that satisfies
$\nu_{\hat{\imath}}-\nu_* \lesssim F^{-1}(\frac{\log(2/\delta)}{n}) - \nu_*$ with probability at least $1-\delta$. 
This guarantee is similar to the result in Proposition~\ref{prop:uniform_sampling_bound}. However, \succhalv achieves its guarantee as long as\footnote{We say $f \simeq g$ if there exist constants $c,c'$ such that $c g(x) \leq f(x) \leq c' g(x)$.}
\begin{align} \label{eq:succ_halv_loose}
B \simeq & \log_2(n) \left[ \log(1/\delta) \gamma^{-1}\left(  F^{-1}(\tfrac{\log(1/\delta)}{n}) -\nu_* \right) +  n \int_{\frac{\log(1/\delta)}{n} }^1 \gamma^{-1}(F^{-1}(t) - \nu_{*}) dt \right]\,,
\end{align}
and this sample complexity may be substantially smaller than the budget required by uniform allocation shown in  \eqref{eq:unif_budget} of Proposition~\ref{prop:uniform_sampling_bound}.
Essentially, the first term in \eqref{eq:succ_halv_loose} represents the budget allocated to the constant number of arms with limits $\nu_i \approx F^{-1}(\frac{\log(1/\delta)}{n})$ while the second term describes the number of times the sub-optimal arms are sampled before discarded.
The next section uses a particular parameterization for $F$ and $\gamma$ to help better
illustrate the difference between the sample complexity of uniform allocation (Equation \ref{eq:unif_budget}) versus that of \succhalv (Equation \ref{eq:succ_halv_loose}).

\subsubsection{A Parameterization of $F$ and $\gamma$ for Interpretability}
\label{sec:param}
To gain some intuition and relate the results back to the existing literature, we make explicit parametric
assumptions on $F$ and $\gamma$. We stress that all of our results
hold for general $F$ and $\gamma$ as previously stated, and this
parameterization is simply a tool to provide intuition. 
First assume that there exists a constant $\alpha>0$ such that
\begin{equation} \label{eq:gamma_assumption}
\gamma(j) \simeq \left( \frac{1}{j} \right)^{1/\alpha} .
\end{equation} 
Note that a large value of $\alpha$ implies that the convergence of $\ell_{i,k} \rightarrow \nu_i$ is very slow. 

We will consider two possible parameterizations of $F$.  First, assume there exists positive constants $\beta$ such that
\begin{equation} \label{eq:F_beta_assumption}
F(x) \simeq  \begin{cases}
 (x-\nu_*)^\beta  & \text{ if } x\geq\nu_* \\[6pt]
 0 & \text{ if } x < \nu_* 
 \end{cases}\\.
\end{equation} 
Here, a large value of $\beta$ implies that it is very rare to draw a limit close to the optimal value $\nu_*$. The same model was studied in \cite{carpentier2015simple}.
Fix some $\Delta > 0$. 
As discussed in the preceding section, if $n = \frac{\log(1/\delta)}{F(\nu_* + \Delta)} \simeq \Delta^{-\beta} \log(1/\delta)$ arms are drawn from $F$ then with probability at least $1-\delta$ we have $\min_{i=1,\dots,n} \nu_i \leq \nu_* + \Delta$. 
Predictably, both uniform allocation and \succhalv output a $\nu_{\hat{\imath}}$ that satisfies 
$\nu_{\hat{\imath}}- \nu_* \lesssim \left( \frac{  \log(1/\delta) }{n} \right)^{1/\beta}$
with probability at least $1-\delta$ provided their measurement budgets are large enough.
Thus, if $n \simeq \Delta^{-\beta} \log(1/\delta)$ and the measurement budgets of the uniform allocation (Equation \ref{eq:unif_budget}) and \succhalv (Equation \ref{eq:succ_halv_loose}) satisfy
\begin{align*}
\text{Uniform allocation}  \quad\quad B &\simeq \Delta^{-(\alpha+\beta)}\log(1/\delta) \\
\text{\succhalv}  \quad\quad B
&\simeq \log_2(\Delta^{-\beta} \log(1/\delta)) \left[ \Delta^{-\alpha} \log(1/\delta) +  \frac{ \Delta^{-\beta} - \Delta^{- \alpha}}{1-\alpha/\beta} \log(1/\delta) \right] \\
&\simeq \log(\Delta^{-1} \log(1/\delta)) \log(\Delta^{-1})  \ \Delta^{-\max\{\beta,\alpha\}} \log(1/\delta) 
\end{align*} 
then both also satisfy $\nu_{\hat{\imath}}- \nu_* \lesssim \Delta$ with probability at least $1-\delta.$\footnote{These quantities are intermediate results in the proofs of the theorems of Section~\ref{sec:theory:guarantees_infinite}.}
\succhalv's budget scales like $\Delta^{-\max\{\alpha,\beta\}}$, which can be significantly smaller than the uniform allocation's budget of $ \Delta^{-(\alpha+\beta)}$.
However, because $\alpha$ and $\beta$ are unknown in practice, neither method knows how to choose the optimal $n$ or $B$ to achieve this $\Delta$ accuracy.
In Section~\ref{sec:theory:guarantees_infinite}, we show how $\ouralg$ addresses this issue.

The second parameterization of $F$ is the following discrete distribution:
\begin{equation} \label{eq:F_discrete_assumption}
F(x) = \frac{1}{K} \sum_{j=1}^K \1\{ x \leq \mu_j \} \quad \quad \text{ with } \quad\Delta_j := \mu_j-\mu_1
\end{equation}
for some set of unique scalars $\mu_1 < \mu_2 < \dots < \mu_K$.
Note that by letting $K \rightarrow \infty$ this discrete CDF can approximate any piecewise-continuous CDF to arbitrary accuracy.
In particular, this model can have multiple means take the same value so that $\alpha$ mass is on $\mu_1$ and $1-\alpha$ mass is on $\mu_2 > \mu_1$, capturing the stochastic infinite-armed bandit model of \cite{jamieson2016power}. 
In this setting, both uniform allocation and \succhalv output a $\nu_{\hat{\imath}}$ that is within the top $\frac{  \log(1/\delta) }{n}$ fraction of the $K$ arms with probability at least $1-\delta$ if their budgets are sufficiently large.
Thus,  let $q>0$ be such that $n \simeq q^{-1} \log(1/\delta)$.  Then, if the measurement budgets of the uniform allocation (Equation \ref{eq:unif_budget}) and \succhalv (Equation \ref{eq:succ_halv_loose}) satisfy
\begin{align*}
\text{Uniform allocation}  \quad B &\simeq \log(1/\delta) \begin{cases}
 K \displaystyle\max_{j=2,\dots,K} \Delta_j^{-\alpha}  & \text{ if } q = 1/K \\[6pt]
  q^{-1} \Delta_{ \lceil qK  \rceil}^{-\alpha} & \text{ if } q > 1/K 
 \end{cases}\\
\text{\succhalv}  \quad B
&\simeq  \log(q^{-1}\log(1/\delta))  \log(1/\delta) \begin{cases}
  \Delta_2^{-\alpha} + \displaystyle\sum_{j=2}^K \Delta_j^{-\alpha} & \text{ if } q = 1/K \\
  \Delta_{ \lceil qK  \rceil}^{-\alpha} + \tfrac{1}{qK} \displaystyle\sum_{j= \lceil qK  \rceil}^K \Delta_j^{-\alpha} & \text{ if } q > 1/K,
  \end{cases}
\end{align*}
an arm that is in the best $q$-fraction of arms is returned, i.e.\ $\hat{\imath}/K \approx  q $ and $\nu_{\hat{\imath}} - \nu_* \lesssim \Delta_{\lceil \max\{2,qK\} \rceil}$, with probability at least $1-\delta$.
This shows that the average resource per arm for uniform allocation is that required to distinguish the top $q$-fraction from the best, while that for \succhalv is a small multiple of the average resource required to distinguish an arm from the best; the difference between the max and the average can be very large in practice.   
We remark that the value of $\epsilon$ in Corollary~\ref{cor:succ_halv_rand_arms_infinite} is carefully chosen to make the \succhalv budget and guarantee work out.
Also note that one would never take $q < 1/K$ because $q=1/K$ is sufficient to return the best arm.

\subsubsection{\ouralg Guarantees}\label{sec:theory:guarantees_infinite}
The \ouralg algorithm of Figure~\ref{alg:hyperband_infinite} addresses the tradeoff between the number of arms $n$ versus the average number of times each one is pulled $B/n$ by performing a two-dimensional version of the so-called ``doubling trick.'' 
For each fixed $B$, we non-adaptively search a predetermined grid of values of $n$ spaced geometrically apart so that the incurred loss of identifying the ``best'' setting takes a budget no more than $\log(B)$ times the budget necessary if the best setting of $n$ were known ahead of time. 
Then, we successively double $B$ so that the cumulative number of measurements needed to arrive at the necessary $B$ is no more than $2B$.
The idea is that even though we do not know the optimal setting for $B,n$ to achieve some desired error rate, the hope is that by trying different values in a particular order, we will not waste too much effort.

Fix $\delta \in (0,1)$.
For all $(k,s)$ pairs defined in the \ouralg algorithm of Figure~\ref{alg:hyperband_infinite}, let $\delta_{k,s} = \frac{\delta}{2k^3}$. 
For all $(k,s)$ define 
\begin{align*}
\mathcal{E}_{k,s} := \{ B_{k,s} > 4 \lceil \log_2(n_{k,s}) \rceil \mathbf{H}(F,\gamma,n_{k,s},\delta_{k,s})\} = \{2^k > 4 s \mathbf{H}(F,\gamma,2^s,2k^3/\delta)\}
\end{align*}
Then by Corollary~\ref{cor:succ_halv_rand_arms_infinite} we have 
\begin{align*}
\P\left( \bigcup_{k=1}^\infty \bigcup_{s=1}^k \{ \nu_{\hat{\imath}_{k,s}} - \nu_* > 3 \big(F^{-1}(\tfrac{\log(4k^3/\delta)}{2^s}) - \nu_*\big) \} \cap \mathcal{E}_{k,s} \right) \leq \sum_{k=1}^\infty \sum_{s=1}^k \delta_{k,s} = \sum_{k=1}^\infty \frac{\delta}{2k^2} \leq \delta.
\end{align*}
For sufficiently large $k$ we will have $\bigcup_{s =1}^k \mathcal{E}_{k,s} \neq \emptyset$, so assume $B= 2^k$ is sufficiently large. 
Let $\hat{\imath}_B$ be the empirically best-performing arm output from  \succhalv of round $k_B=\lfloor\log_2(B)\rfloor$ of \ouralg of Figure~\ref{alg:hyperband_infinite} and let $s_B \leq k_B$ be the largest value such that $\mathcal{E}_{k_B,s_B}$ holds.
Then 
\begin{align*}
\nu_{\hat{\imath}_{B}} - \nu_* &\leq 3\big(F^{-1}(\frac{\log(4 \lfloor \log_2(B) \rfloor^3 /\delta )}{2^{s_B}}) - \nu_*\big) + \gamma( \lfloor \tfrac{ 2^{\lfloor \log_2(B) \rfloor-1} }{\lfloor \log_2(B) \rfloor} \rfloor ) .
\end{align*}
Also note that on stage $k$ at most $\sum_{i=1}^k i B_{i,1} \leq k \sum_{i=1}^k B_{i,1} \leq 2 k B_{k,s} = 2 \log_2(B_{k,s}) B_{k,s}$ total samples have been taken.
While this guarantee holds for general $F,\gamma$, the value of $s_B$, and consequently the resulting bound, is difficult to interpret. The following corollary considers the $\beta,\alpha$ parameterizations of $F$ and $\gamma$, respectively, of Section~\ref{sec:param} for better interpretation.

\begin{theorem}\label{thm:hyperband_infinite_main}
Assume that Assumptions 1 and 2 of Section~\ref{sec:theory_assumptions} hold and that the sampled loss sequences obey the parametric assumptions of Equations~\ref{eq:gamma_assumption} and \ref{eq:F_beta_assumption}. 
Fix $\delta \in (0,1)$.
For any $T \in \mathbb{N}$, let $\hat{\imath}_T$ be the empirically best-performing arm output from  \succhalv from the last round $k$ of \ouralg of Figure~\ref{alg:hyperband_infinite} after exhausting a total budget of $T$ from all rounds, then
\begin{align*}
\nu_{\hat{\imath}_T} - \nu_* \leq c \left( \frac{ \overline\log(T)^3 \overline\log(\log(T)/\delta) }{T} \right)^{1/\max\{\alpha,\beta\}}
\end{align*} 
for some constant $c = \exp(O(\max\{\alpha,\beta\}))$ where $\overline\log(x) = \log(x )\log\log(x)$.
\end{theorem}

By a straightforward modification of the proof, one can show that if uniform allocation is used in place of \succhalv in \ouralg, the uniform allocation version achieves $\nu_{\hat{\imath}_T} - \nu_* \leq c \left( \frac{ \log(T) \overline\log(\log(T)/\delta) }{T} \right)^{1/(\alpha+\beta)}$. 
We apply the above theorem to the stochastic infinite-armed bandit setting in the following corollary.

\begin{corollary}\label{cor:hyperband_stochastic_F}[Stochastic Infinite-armed Bandits]
For any step $k,s$ in the infinite horizon \ouralg algorithm with $n_{k,s}$ arms drawn, consider the setting where the $j$th pull of the $i$th arm results in a stochastic loss $Y_{i,j} \in [0,1]$ such that $\E[Y_{i,j}] = \nu_i$ and $\P( \nu_i - \nu_* \leq \epsilon ) = c_1^{-1} \epsilon^\beta$. If $\ell_j(i) = \frac{1}{j} \sum_{s=1}^j Y_{i,s}$ then with probability at least $1-\delta/2$ we have $\forall k \geq 1, 0 \leq s \leq k, 1\leq i \leq n_{k,s},1 \leq j \leq B_k$, 
\begin{align*}
 |\nu_i - \ell_{i,j}| \leq \sqrt{\tfrac{\log(B_k n_{k,s} / \delta_{k,s})}{2j}} \leq \sqrt{ \log(\tfrac{16 B_k}{\delta}) } \left( \tfrac{2}{j} \right)^{1/2}.
\end{align*} 
Consequently, if after $B$ total pulls we define $\widehat{\nu}_B$ as the mean of the empirically best arm output from the last fully completed round $k$, then with probability at least $1-\delta$
\begin{align*}
\widehat{\nu}_B - \nu_* \leq \mathrm{polylog}(B/\delta) \max\{ B^{-1/2},  B^{-1/\beta} \}.
\end{align*} 
\end{corollary}

The result of this corollary matches the anytime result of Section~4.3 of \citet{carpentier2015simple} whose algorithm was built specifically for the case of stochastic arms and the $\beta$ parameterization of $F$ defined in \eqref{eq:F_beta_assumption}. Notably, this result also matches the {\em lower bounds} shown in that work up to poly-logarithmic factors, revealing that \ouralg is nearly tight for this important special case.
However, we note that this earlier work has a more careful analysis for the fixed budget setting.

\begin{theorem}\label{thm:hyperband_infinite_alt}
Assume that Assumptions 1 and 2 of Section~\ref{sec:theory_assumptions} hold and that the sampled loss sequences obey the parametric assumptions of Equations~\ref{eq:gamma_assumption} and \ref{eq:F_discrete_assumption}. 
For any $T \in \mathbb{N}$, let $\hat{\imath}_T$ be the empirically best-performing arm output from  \succhalv from the last round $k$ of \ouralg of Figure~\ref{alg:hyperband_infinite} after exhausting a total budget of $T$ from all rounds.
Fix $\delta \in (0,1)$ and $q \in (1/K,1)$ and let
	$z_q = \log( q^{-1} ) (\Delta_{ \lceil \max\{2,q K\} \rceil }^{-\alpha}  + \frac{ 1}{q K} \sum_{i=\lceil \max\{2,q K\} \rceil}^K \Delta_i^{-\alpha} )$.
	Once $T = \widetilde{\Omega}\left( z_q \log( z_q ) \log( 1/\delta) \right)$ total pulls have been made by \ouralg we have $\widehat{\nu}_T - \nu_* \leq \Delta_{\lceil\max\{2,qK\}\rceil}$ with probability at least $1-\delta$ where $\widetilde{\Omega}(\cdot)$ hides $\log\log(\cdot)$ factors.
\end{theorem}

Appealing to the stochastic setting of Corollary~\ref{cor:hyperband_stochastic_F} so that $\alpha=2$, we conclude that the sample complexity sufficient to identify an arm within the best $q$ proportion with probabiltiy $1-\delta$, up to log factors, scales like $\log(1/\delta) \log( q^{-1} ) (\Delta_{ \lceil q K \rceil }^{-\alpha}  + \frac{ 1}{q K} \sum_{i=\lceil q K \rceil}^K \Delta_i^{-\alpha} )$. 
One may interpret this result as an extension of the distribution-dependent pure-exploration results of \cite{bubeck2009pure}; but in our case, our bounds hold when the number of pulls is potentially much smaller than the number of arms $K$.
When $q=1/K$ this implies that the best arm is identified with about $\log(1/\delta) \log(K) \{\Delta_2^{-2} +\sum_{i=2}^K \Delta_i^{-2}\}$ which matches known upper bounds \cite{karnin2013almost,jamieson2014lil} and lower bounds \cite{kaufmann2015complexity} up to $\log$ factors. 
Thus, for the stochastic $K$-armed bandit problem \ouralg recovers many of the known sample complexity results up to $\log$ factors.

\subsection{Finite Horizon Setting ($\Rmax<\infty$)}
In this section we analyze the algorithm described in Section~\ref{sec:algorithm}, i.e.\ finite horizon \ouralg. We present similar theoretical guarantees as in Section~\ref{ssec:theory_infinite} for infinite horizon \ouralg, and fortunately much of the analysis will be recycled. We state the finite horizon version of the \succhalv and \ouralg algorithms in Figure~\ref{alg:hyperband_finite}.

\begin{figure}[ht!]
\centerline{
\fbox{\parbox[b]{.8\textwidth}{{\underline{\succhalv} (Finite horizon) } \\[2pt] \small
{\bf input}: Budget $B$, and $n$ arms where $\ell_{i,k}$ denotes the $k$th loss from the $i$th arm, maximum size $\Rmax$, $\eta \geq 2$ ($\eta = 3$ by default). \\[3pt]
{\bf Initialize}: $S_0 = [n]$, $s = \min \{ t \in \mathbb{N} : nR(t+1)\eta^{-t} \leq B, t \leq \log_\eta(\min\{R,n\})\}$. \\[3pt] 
\textbf{For} $k=0,1,\dots,s$ \\ [3pt]
\indent \hspace{.4cm} Set $n_k = \lfloor n \eta^{-k} \rfloor$, $r_k = \lfloor R \eta^{k-s} \rfloor$ \\[3pt]
\indent \hspace{.4cm} Pull each arm in $S_k$ for $r_k$ times.\\[3pt]
\indent \hspace{.4cm} Keep the best $\lfloor n\eta^{-(k+1)} \rfloor$ arms in terms of the $r_k$th observed loss as $S_{k+1}$.\\[3pt]
{\bf Output} : $\hat{\imath},\ell_{\hat{\imath},R}$ where $\hat{\imath}=\arg\min_{i \in S_{s+1}} \ell_{i,R}$
}}
}
\centering
\fbox{\parbox[b]{.8\textwidth}{{\underline{\ouralg} (Finite horizon)}  \\[2pt]  \small
\textbf{Input:} Budget $B$, maximum size $\Rmax$, $\eta \geq 2$ ($\eta = 3$ by default) \\[3pt]
\textbf{Initialize:} $s_{\max} = \lfloor \log(\Rmax) / \log(\eta) \rfloor$ \\[3pt]
\textbf{For} $k=1,2,\dots$\\[3pt]
\indent \hspace{.4cm} \textbf{For} $s=s_{\max},s_{\max}-1,\dots,0$\\[3pt]
\indent \hspace{.8cm} $B_{k,s}=2^k$, $n_{k,s} = \lceil \frac{2^k \eta^s}{\Rmax(s+1)}\rceil$ \\[3pt]
\indent \hspace{.8cm} $\hat{\imath}_{s}, \ell_{\hat{\imath}_s,R} \leftarrow$ \succhalv($B_{k,s}$,$n_{k,s}$,$R$,$\eta$) 
}}
\caption{The finite horizon \succhalv and \ouralg algorithms are inspired by their infinite horizon counterparts of Figure~\ref{alg:hyperband_infinite} to handle practical constraints. \ouralg calls \succhalv as a subroutine.}
\label{alg:hyperband_finite}
\end{figure} 

The finite horizon setting differs in two major ways. First, in each bracket at least one arm will be pulled $\Rmax$ times, but no arm will be pulled more than $\Rmax$ times. 
Second, the number of brackets, each representing \succhalv with a different tradeoff between $n$ and $B$, is fixed at $\log_\eta(\Rmax)+1$.
Hence, since we are sampling sequences randomly i.i.d., increasing $B$ over time would just multiply the number of arms in each bracket by a constant, affecting performance only by a small constant.

\begin{theorem} \label{thm:succ_halv_finite}
Fix $n$ arms. Let $\nu_i = \ell_{i,\Rmax}$ and assume $\nu_1 \leq \dots \leq \nu_n$. For any $\epsilon >0$ let 
\begin{align*}
z_{SH} &=  \eta (\log_\eta(\Rmax) + 1) \Big[n + \sum_{i=1}^n \min\big \{R, \gamma^{-1} \left( \max\left\{ \tfrac{\epsilon}{4} , \tfrac{\nu_{ i }- \nu_{1}}{2} \right\} \right) \big\}  \Big]
\end{align*} 
If the Successive Halving algorithm of Figure~\ref{alg:hyperband_finite} is run with any budget $B \geq z_{SH}$ then an arm $\hat{\imath}$ is returned that satisfies $\nu_{\hat{\imath}} - \nu_1 \leq \epsilon/2$. 
\end{theorem}

Recall that $\gamma(\Rmax) = 0$ in this setting and by definition $\sup_{y \geq 0} \gamma^{-1}(y) \leq R$. 
Note that Lemma~\ref{lem:gamma_sum_bound_rand_arms} still applies in this setting and just like above we obtain the following corollary.

\begin{corollary}
\label{cor:succ_halv_rand_arms_finite}
Fix $\delta \in (0,1)$ and $\epsilon \geq 4(F^{-1}(\tfrac{\log(2/\delta)}{n}) - \nu_*)$. 
Let $\mathbf{H}(F,\gamma,n,\delta,\epsilon)$ be as defined in Lemma~\ref{lem:gamma_sum_bound_rand_arms} and $B = \eta \log_\eta(\Rmax) (n+\max\{R, \mathbf{H}(F,\gamma,n,\delta,\epsilon)\})$. 
If the \succhalv algorithm of Figure~\ref{alg:hyperband_finite} is run with the
specified $B$ and $n$ arm configurations drawn randomly according to $F$ then
an arm $\hat{\imath} \in [n]$ is returned such that with probability at least
$1-\delta$ we have 
$\nu_{\hat{\imath}} - \nu_* \leq \big(F^{-1}(\frac{\log(2/\delta)}{n}) - \nu_*\big) + \epsilon/2$.
In particular, if $B = 4 \lceil \log_2(n) \rceil  \mathbf{H}(F,\gamma,n,\delta)$ and $\epsilon = 4(F^{-1}(\tfrac{\log(2/\delta)}{n}) - \nu_*)$ then $\nu_{\hat{\imath}} - \nu_* \leq 3\big(F^{-1}(\frac{\log(2/\delta)}{n}) - \nu_*\big)$ with probability at least $1-\delta$.
\end{corollary}

As in Section~\ref{sec:param} we can apply the $\alpha,\beta$ parameterization for interpretability, with the added constraint that $\sup_{y \geq 0} \gamma^{-1}(y) \leq R$ so that $\gamma(j) \simeq \1_{j < \Rmax} \left( \frac{1}{j} \right)^{1/\alpha}$.
Note that the approximate sample complexity of \succhalv given in \eqref{eq:succ_halv_loose} is still valid for the finite horizon algorithm.

Fixing some $\Delta >0$, $\delta \in (0,1)$, 
 and applying the parameterization of~\eqref{eq:F_beta_assumption} we recognize that if $n \simeq \Delta^{-\beta} \log(1/\delta)$ and the sufficient sampling budgets (treating $\eta$ as an absolute constant) of the uniform allocation (Equation \ref{eq:unif_budget}) and \succhalv (\eqref{eq:succ_halv_loose}) satisfy
\begin{align*}
\text{Uniform allocation}  \quad\quad B &\simeq R \Delta^{-\beta} \log(1/\delta) \\
\text{\succhalv}  \quad\quad B 
 &\simeq \log(\Delta^{-1} \log(1/\delta)) \log(1/\delta) \left[ R + \Delta^{-\beta} \frac{1-(\alpha/\beta)R^{1-\beta/\alpha}}{1-\alpha/\beta}  \right]
 \end{align*}
then both also satisfy $\nu_{\hat{\imath}}- \nu_* \lesssim \Delta$ with probability at least $1-\delta$. 
Recalling that a larger $\alpha$ means slower convergence and that a larger $\beta$ means a greater difficulty of sampling a good limit,
note that when $\alpha/\beta < 1$ the budget of \succhalv behaves like $R+\Delta^{-\beta}\log(1/\delta)$ but as $\alpha/\beta \rightarrow \infty$ the budget asymptotes to $R \Delta^{-\beta} \log(1/\delta)$.

We can also apply the discrete-CDF parameterization of \eqref{eq:F_discrete_assumption}.
For any $q \in (0,1)$, if $n \simeq q^{-1} \log(1/\delta)$ and the measurement budgets of the uniform allocation (Equation \ref{eq:unif_budget}) and \succhalv (Equation \ref{eq:succ_halv_loose}) satisfy
\begin{align*}
\text{Uniform allocation:} \quad\quad\quad& B \simeq \log(1/\delta) \begin{cases}
 K \min\left\{ \Rmax, \displaystyle\max_{j=2,\dots,K} \Delta_j^{-\alpha} \right\} & \text{ if } q = 1/K \\[6pt]
  q^{-1} \min\{ \Rmax, \Delta_{ \lceil qK  \rceil}^{-\alpha} \} & \text{ if } q > 1/K 
 \end{cases}\\
\text{\succhalv:} \quad\quad\quad&  \\
B \simeq \log(q^{-1} \log(1/\delta)) \log(1/\delta) &\begin{cases}
  \min\{ \Rmax, \Delta_2^{-\alpha} \} + \displaystyle\sum_{j=2}^K \min\{ \Rmax, \Delta_j^{-\alpha} \} & \text{ if } q = 1/K \\
  \min\{ \Rmax,\Delta_{ \lceil qK  \rceil}^{-\alpha}\} + \tfrac{1}{qK} \displaystyle\sum_{j= \lceil qK  \rceil}^K \min\{ \Rmax, \Delta_j^{-\alpha}\} & \text{ if } q > 1/K 
  \end{cases}
\end{align*}
then an arm that is in the best $q$-fraction of arms is returned, i.e.\ $\hat{\imath}/K \approx  q $ and $\nu_{\hat{\imath}} - \nu_* \lesssim \Delta_{\lceil \max\{2,qK\} \rceil}$, with probability at least $1-\delta$.
Once again we observe a stark difference between uniform allocation and \succhalv, particularly when $\Delta_j^{-\alpha} \ll \Rmax$ for many values of $j \in \{1,\dots,n\}$.

Armed with Corollary~\ref{cor:succ_halv_rand_arms_finite}, all of the discussion of Section~\ref{sec:theory:guarantees_infinite} preceding Theorem~\ref{thm:hyperband_infinite_main} holds for the finite case ($\Rmax < \infty$) as well. 
Predictably analogous theorems also hold for the finite horizon setting, but their specific forms (with the $\text{polylog}$ factors) provide no additional insights beyond the sample complexities sufficient for \succhalv to succeed, given immediately above.

It is important to note that in the finite horizon setting, for all sufficiently large $B$ (e.g. $B>3R$) and all distributions $F$, the budget $B$ of \succhalv should scale {\em linearly} with $n \simeq \Delta^{-\beta} \log(1/\delta)$ as $\Delta \rightarrow 0$. 
Contrast this with the infinite horizon setting in which the ratio of $B$ to $n$ can become unbounded based on the values of $\alpha,\beta$ as $\Delta \rightarrow 0$.
One consequence of this observation is that in the finite horizon setting it suffices to set $B$ large enough to identify an $\Delta$-good arm with just constant probability, say $1/10$, and then repeat \succhalv $m$ times to boost this constant probability to probability $1-(\frac{9}{10})^{m}$.
While in this theoretical treatment of \ouralg we grow $B$ over time, in practice we recommend fixing $B$ as a multiple of $R$ as we have done in Section~\ref{sec:algorithm}.  The fixed budget version of finite horizon \ouralg is more suitable for practical application due to the constant time, instead of exponential time, between configurations trained to completion in each outer loop.


\section{Conclusion}\label{sec:extensions}
We conclude by discussing three potential extensions related to parallelizing \ouralg for distributed computing, adjusting for training methods with different convergence rates, and combining \ouralg with non-random sampling methods.  

 \textbf{Distributed implementations}. \ouralg has the potential to be parallelized since arms are independent and sampled randomly.  The most straightforward parallelization scheme is to distribute individual brackets of \succhalv to different machines.  This can be done asynchronously and as machines free up, new brackets can be launched with a different set of arms.  One can also parallelize a single bracket so that each round of \succhalv runs faster.  One drawback of this method is that if $\Rmax$ can be computed on one machine, the number of tasks decreases exponentially as arms are whittled down so a more sophisticated job priority queue must be managed. Devising parallel generalizations of \ouralg that efficiently leverage massive distributed clusters while minimizing overhead costs is an interesting avenue for future work.
  
\textbf{Adjusting for different convergence rates}. A second open challenge involves generalizing the ideas behind \ouralg to settings where configurations have drastically differing convergence rates.  Configurations can have different convergence rates 
if they have hyperparameters that impact convergence (e.g., learning rate decay for SGD or neural networks with differing numbers of layers or hidden units), and/or if they correspond to different model families (e.g., deep networks versus decision trees).
The core issue arises when configurations with drastically slower convergence rates ultimately result in better models. 
To address these issues, we should be able to adjust the resources allocated to each configuration so that a fair comparison can be made at the time of elimination.

\textbf{Incorporating non-random sampling}.
Finally, \ouralg can benefit from different sampling schemes aside from simple random search.  Quasi-random methods like Sobol or latin hypercube which were studied in \citet{Bergstra2012} may improve the performance of \ouralg by giving better coverage of the search space.  Alternatively, meta-learning can be used to define intelligent priors informed by previous experimentation~\citep{Feurer2015}. Finally, as mentioned in Section~\ref{sec:related}, exploring ways to combine \ouralg with adaptive configuration selection strategies is a very promising future direction.

\acks{KJ is supported by ONR awards N00014-15-1-2620 and N00014-13-1-0129. AT is supported in part by a Google Faculty Award and an AWS in Education Research Grant award. }

\clearpage
\appendix
\section{Additional Experimental Results}\label{add_results}
Additional details for experiments presented in Section~\ref{sec:algorithm} and \ref{sec:experiments} are provided below.
\subsection{LeNet Experiment}
The search space for the LeNet example discussed in Section~\ref{ssec:lenet} is shown in Table~\ref{tab:lenet}.
\begin{table}[h!]
\begin{center}
\begin{tabular}{| l | l | l | l |}
 \hline
Hyperparameter & Scale & Min & Max \\ 
\hline
Learning Rate & log & 1e-3 & 1e-1 \\
\hline
Batch size & log & 1e1 & 1e3 \\
\hline
Layer-2 Num Kernels (k2) & linear & 10 & 60 \\
\hline
Layer-1 Num Kernels (k1) & linear & 5 & k2\\
\hline
\end{tabular}
\caption{Hyperparameter space for the LeNet application of Section~\ref{ssec:lenet}. Note that the number of kernels in Layer-1 is upper bounded by the number of kernels in Layer-2.}
\label{tab:lenet}
\end{center}
\end{table}

\subsection{Experiments Using Alex Krizhevsky's CNN Architecture}
For the experiments discussed in Section~\ref{ssec:cnn}, the exact architecture used is the $18\%$ model provided on cuda-convnet for CIFAR-10.\footnote{The model specification is available at \url{http://code.google.com/p/cuda-convnet/}.}  

\textbf{Search Space:} The search space used for the experiments is shown in Table \ref{tab:cnn_search_space}.  The learning rate reductions hyperparameter indicates how many times the learning rate was reduced by a factor of 10 over the maximum iteration window.  For example, on CIFAR-10, which has a maximum iteration of 30,000, a learning rate reduction of 2 corresponds to reducing the learning every 10,000 iterations, for a total of 2 reductions over the 30,000 iteration window.  All hyperparameters, with the exception of the learning rate decay reduction, overlap with those in \citet{Snoek2012}.  Two hyperparameters in \citet{Snoek2012} were excluded from our experiments: (1) the width of the response normalization layer was excluded due to limitations of the Caffe framework and (2) the number of epochs was excluded because it is incompatible with dynamic resource allocation. 

\begin{table}[h]
\begin{center}
\begin{tabular}{lcccc}
Hyperparameter  & Scale & Min & Max\\
\hline
\emph{Learning Parameters} \\
\hspace{0.25cm}Initial Learning Rate  & log & $5*10^{-5}$ & 5\\
\hspace{0.25cm} Conv1 $L_2$ Penalty  & log & $5*10^{-5}$ & 5\\
\hspace{0.25cm} Conv2 $L_2$ Penalty  & log & $5*10^{-5}$ & 5\\
\hspace{0.25cm} Conv3 $L_2$ Penalty  & log & $5*10^{-5}$ & 5\\
\hspace{0.25cm} FC4 $L_2$ Penalty  & log & $5*10^{-3}$ & 500\\
\hspace{0.25cm}Learning Rate Reductions & integer & 0 & 3\\
\emph{Local Response Normalization}\\
\hspace{0.25cm}Scale  & log & $5*10^{-6}$ & 5 \\
\hspace{0.25cm}Power  & linear & 0.01 & 3
\end{tabular}
\caption{Hyperparameters and associated ranges for the three-layer convolutional network.}\label{tab:cnn_search_space}
\end{center}
\end{table}

\textbf{Data Splits:}  For CIFAR-10, the training (40,000 instances) and validation (10,000 instances) sets were sampled from data batches 1-5 with balanced classes.  The original test set (10,000 instances) was used for testing.  For MRBI, the training (10,000 instances) and validation (2,000 instances) sets were sampled from the original training set with balanced classes.  The original test set (50,000 instances) was used for testing.  Lastly, for SVHN, the train, validation, and test splits were created using the same procedure as that in \citet{svhnsplit}.  

\textbf{Comparison with Early-Stopping:}  \citet{earlystopping2015} proposed an early-stopping method for neural networks and combined it with SMAC to speed up hyperparameter optimization.  Their method stops training a configuration if the probability of the configuration beating the current best is below a specified threshold.  This probability is estimated by extrapolating learning curves fit to the intermediate validation error losses of a configuration.  If a configuration is terminated early, the predicted terminal value from the estimated learning curves is used as the validation error passed to the hyperparameter optimization algorithm.  Hence, if the learning curve fit is poor, it could impact the performance of the configuration selection algorithm.  While this approach is heuristic in nature, it could work well in practice so we compare \ouralg to SMAC with early termination (labeled SMAC (early) in Figure \ref{fig:cnn_bars}).  We used the conservative termination criterion with default parameters and recorded the validation loss every 400 iterations and evaluated the termination criterion 3 times within the training period (every 8k iterations for CIFAR-10 and MRBI and every 16k iterations for SVHN).\footnote{We used the code provided at https://github.com/automl/pylearningcurvepredictor.}  Comparing the performance by the number of total iterations as mulitple of $\Rmax$ is conservative because it does not  account for the time spent fitting the learning curve in order to check the termination criterion.  

\begin{figure}
\centering
\subfigure[CIFAR-10]
{\includegraphics[width=7cm,trim=0 0 40 0,page=1]{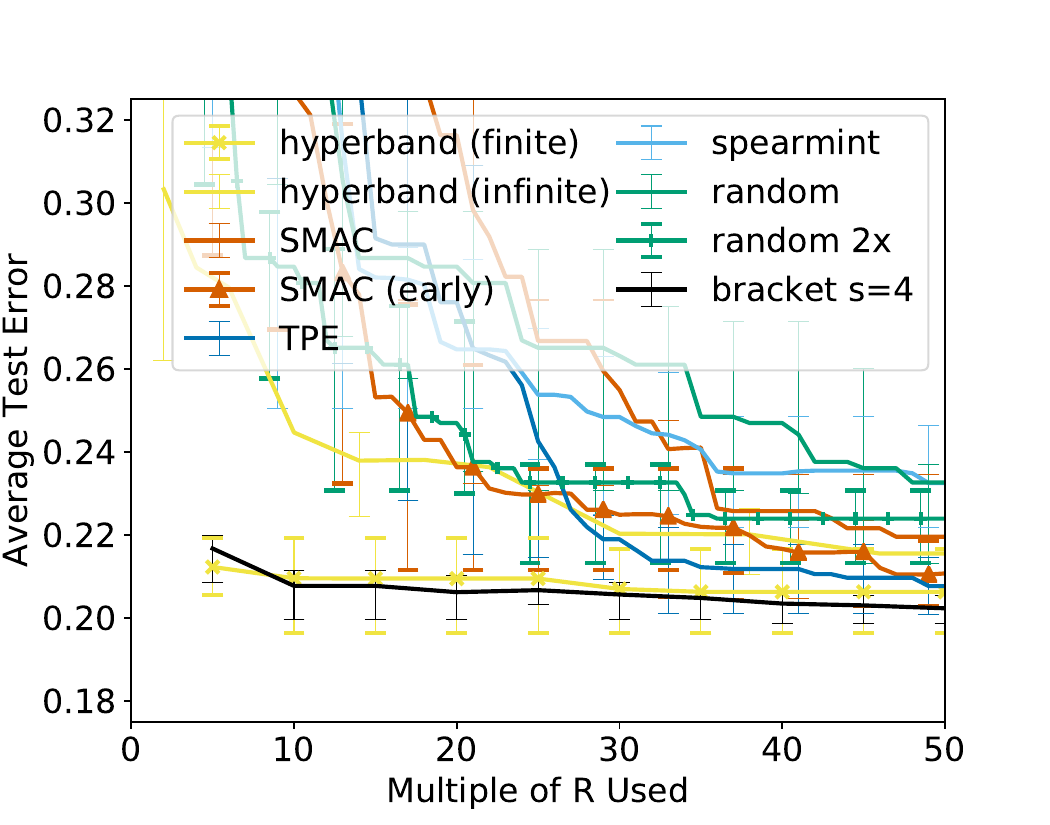}}
\subfigure[MRBI]
{\includegraphics[width=7cm,trim=0 0 40 0,page=2]{error_avg_inf_bars.pdf}}
\subfigure[SVHN]
{\includegraphics[width=7cm,trim=0 0 40 0,page=3]{error_avg_inf_bars.pdf}}
\caption{Average test error across 10 trials is shown in all plots.  Error bars indicate the top and bottom quartiles of the test error corresponding to the model with the best validation error}
\label{fig:cnn_bars}
\end{figure} 
\subsection{117 Data Sets Experiment}
For the experiments discussed in Section~\ref{sssec:117data}, we followed \citet{Feurer2015} and imposed a 3GB memory limit, a 6-minute timeout
for each hyperparameter configuration and a one-hour time window to evaluate
each searcher on each data set. Moreover, we evaluated the performance of each searcher
by aggregating results across all data sets and
reporting the average rank of each method.  
Specifically, the hour training window is divided up into 30 second intervals and, at each time point, the model with the best validation error at that time is used in the calculation of the average error across all trials for each (data set-searcher) pair.  Then, the performance of each searcher is ranked by data set and averaged across all data sets.  All experiments were performed on Google Cloud
Compute n1-standard-1 instances in us-central1-f region with 1 CPU and 3.75GB of memory.

\textbf{Data Splits:}  \citet{Feurer2015} split each data set into 2/3 training and 1/3 test set,
whereas we introduce a validation set to avoid overfitting to the test data. We
also used 2/3 of the data for training, but split the rest of the data into two
equally sized validation and test sets. We reported results on both the
validation and test data. Moreover, we performed 20 trials of each (data set-searcher) pair,
and as in \citet{Feurer2015} we
kept the same data splits across trials, while using a different random seed for each searcher in each trial.

\textbf{Shortcomings of the Experimental Setup:} 
The benchmark contains a large variety of training set sizes and feature dimensions\footnote{Training set size ranges from 670 to 73,962 observations, and number of features ranges from 1 to 10,935.} resulting in random search being able to test 600 configurations on some data sets but just dozens on others. \ouralg  was designed under the implicit assumption that computation scaled at least linearly with the data set size. For very small data sets that are trained in seconds, the initialization overheads dominate the computation and subsampling provides
no computational benefit. In addition, many of
the classifiers and preprocessing methods under consideration return memory
errors as they require storage quadratic in the number of features (e.g.,
covariance matrix) or the number of observations (e.g., kernel methods). These errors
 usually happen immediately (thus wasting little time); however, they often
occur on the full data set and not on subsampled data sets.
A searcher like \ouralg that uses a subsampled data set could spend significant time training on a subsample only to error out when attempting to train it on the full data set. 

\subsection{Kernel Classification Experiments}
Table~\ref{tab:kernel_lsqr} shows the hyperparameters and associated ranges considered in the kernel least squares classification experiment discussed in Section~\ref{ssec:kernel_lsqr}.  
\begin{table}[h!]
\begin{center}
\begin{tabular}{| l | l |l|}
 \hline
Hyperparameter & Type & Values\\ 
\hline
preprocessor & Categorical & min/max, standardize, normalize\\
\hline
kernel & Categorical & rbf, polynomial, sigmoid \\
\hline
C  & Continuous & log $[10^{-3},10^5]$\\
\hline
gamma & Continuous & log $[10^{-5},10]$\\
\hline
degree & if kernel=poly & integer [2, 5]\\
\hline
coef0 & if kernel=poly, sigmoid & uniform [-1.0, 1.0]\\
\hline
\end{tabular}
\caption{Hyperparameter space for kernel regularized least squares classification problem discussed in Section~\ref{ssec:kernel_lsqr}.}
\label{tab:kernel_lsqr}
\end{center}
\end{table}

The cost term C is divided by the number of samples so that the tradeoff between the squared error and the $L_2$ penalty would remain constant as the resource increased (squared error is summed across observations and not averaged).  The regularization term $\lambda$ is equal to the inverse of the scaled cost term C.  Additionally, the average test error with the top and bottom quartiles across 10 trials are show in Figure~\ref{fig:kernel_lsqr_bars}.

\begin{figure}
\centering
  \includegraphics[width=7cm,page=1,trim=10 10 10 10]{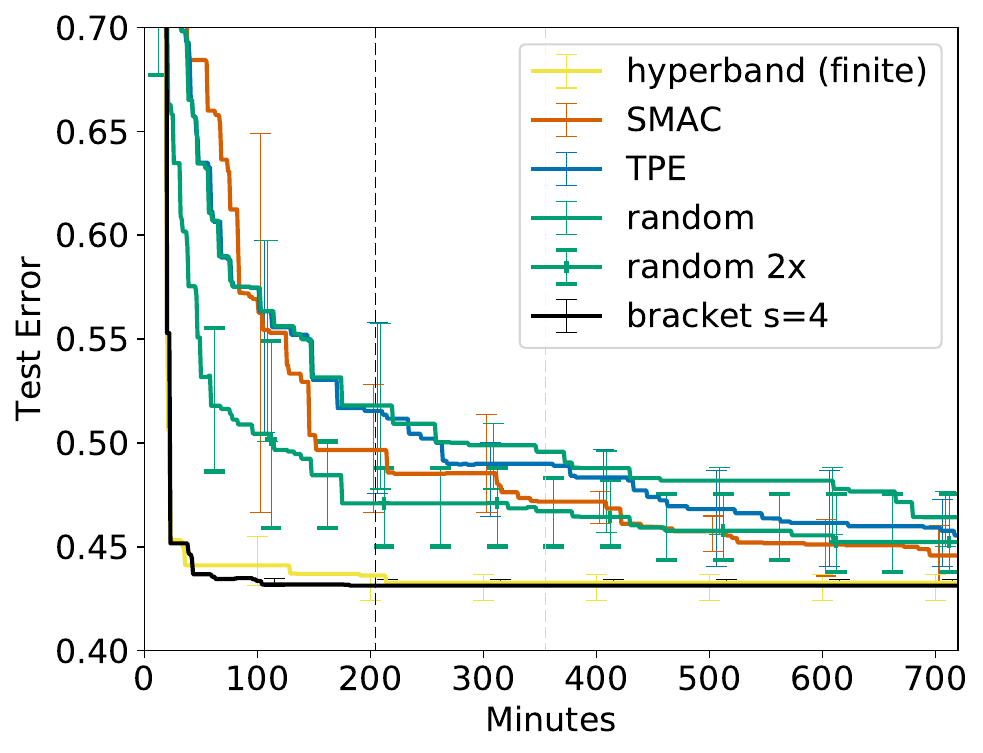}
\caption{Average test error of the best kernel regularized least square classification model found by each searcher on CIFAR-10.  The color coded dashed lines indicate when the last trial of a given searcher finished.  Error bars correspond to the top and bottom quartiles of the test error across 10 trials.} \label{fig:kernel_lsqr_bars}
\end{figure}
Table~\ref{tab:random_features} shows the hyperparameters and associated ranges considered in the random features kernel approximation classification experiment discussed in Section~\ref{ssec:random_features}.
The regularization term $\lambda$ is divided by the number of features so that the tradeoff between the squared error and the $L_2$ penalty would remain constant as the resource increased.  Additionally, the average test error with the top and bottom quartiles across 10 trials are show in Figure~\ref{fig:kernel_approx_bars}.
\begin{table}[h!]
\begin{center}
\begin{tabular}{| l | l |l|}
 \hline
Hyperparameter & Type & Values\\ 
\hline
preprocessor & Categorical & none, min/max, standardize, normalize\\
\hline
$\lambda$  & Continuous & log $[10^{-3},10^{5}]$\\
\hline
gamma & Continuous & log $[10^{-5},10]$\\
\hline
\end{tabular}
\caption{Hyperparameter space for random feature kernel approximation classification problem discussed in Section~\ref{ssec:random_features}.}
\label{tab:random_features}
\end{center}
\end{table}

\begin{figure}
\centering
\includegraphics[width=7cm,page=1,trim=10 10 10 10]{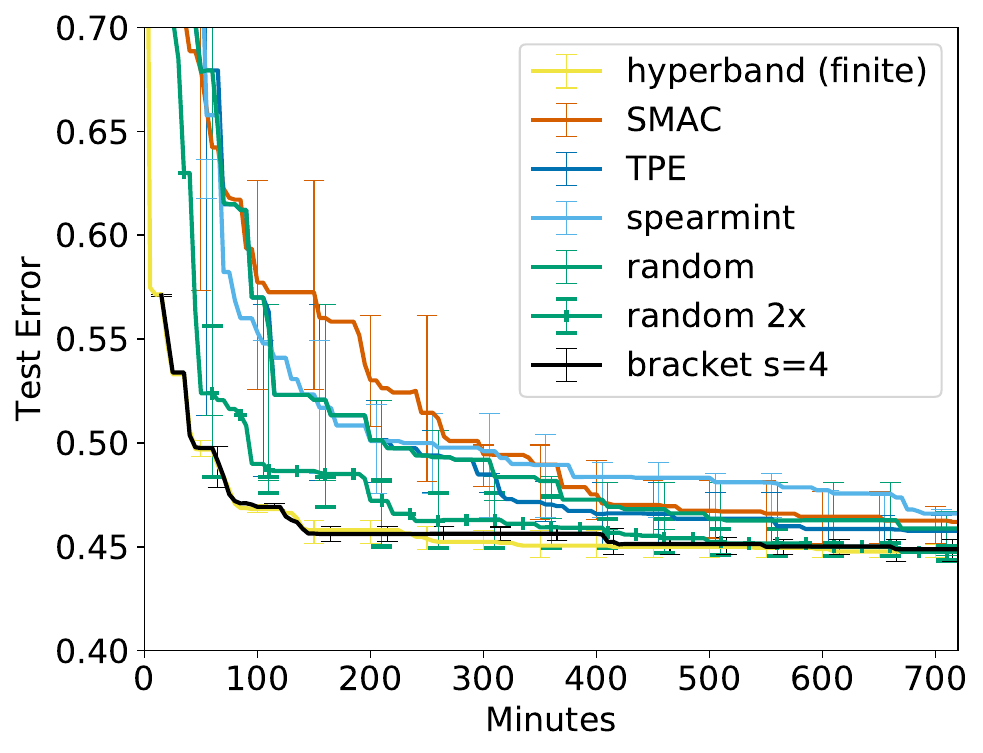}
\caption{Average test error of the best random features model found by each searcher on CIFAR-10.  The test error for \ouralg and bracket $s=4$ are calculated in every evaluation instead of at the end of a bracket.  Error bars correspond to the top and bottom quartiles of the test error across 10 trials.} \label{fig:kernel_approx_bars}
\end{figure}


\section{Proofs}
\label{sec:theory_proofs}
In this section, we provide proofs for the theorems presented in Section~\ref{sec:theory}.
\subsection{Proof of Theorem~\ref{thm:succ_halving}}
\begin{proof}
First, we verify that the algorithm  never takes a total number of samples
that exceeds the budget $B$:
\begin{align*}
\sum_{k=0}^{\lceil \log_2(n) \rceil -1} |S_k|  \left\lfloor \tfrac{ B }{|S_k| \lceil \log(n) \rceil } \right\rfloor \leq \sum_{k=0}^{\lceil \log_2(n) \rceil -1} \tfrac{ B }{\lceil \log(n) \rceil }  \leq B \, .
\end{align*}

For notational ease, let  $\ell_{i,j} := \ell_j(X_i)$. 
Again, for each $i \in [n] := \{1,\dots,n\}$ we assume the limit $\lim_{k \rightarrow \infty} \ell_{i,k}$ exists and is equal to $\nu_i$.
As a reminder, $\gamma : \mathbb{N} \rightarrow \R$  is defined as the pointwise smallest, monotonically decreasing function satisfying
\begin{align}
\max_{i} | \ell_{i,j} - \nu_i | \leq \gamma(j) \ ,  \ \ \forall j \in \mathbb{N}.
\end{align}
Note $\gamma$ is guaranteed to exist by the existence of $\nu_i$ and bounds the deviation from the limit value as the sequence of iterates $j$ increases. 

Without loss of generality, order the terminal losses
so that $\nu_1\leq \nu_2\leq\dots\leq\nu_n$. 
Assume that $B\geq z_{SH}$. 
Then we have for each round $k$
\begin{align*}
    r_k &\geq \frac{B}{|S_k| \lceil \log_2(n) \rceil}-1  \\
    &\geq \frac{2}{|S_k|}\max_{i=2,\dots,n} i\bigg(1+\gamma^{-1}\big(\max\big\{\frac{\epsilon}{4},\frac{\nu_i-\nu_1}{2}\big\}\big)\bigg) -1 \\
    &\geq \frac{2}{|S_{k}|} \ (\lfloor |S_{k}|/2 \rfloor+1) \bigg(1+\gamma^{-1}\big(\max\big\{\frac{\epsilon}{4},\frac{\nu_{\lfloor |S_{k}|/2 \rfloor+1}-\nu_1}{2}\big\}\big)\bigg) -1 \\
    &\geq \bigg(1+\gamma^{-1}\big(\max\big\{\frac{\epsilon}{4},\frac{\nu_{\lfloor |S_{k}|/2 \rfloor+1}-\nu_1}{2}\big\}\big)\bigg) - 1\\
    &= \gamma^{-1}\big(\max\big\{\frac{\epsilon}{4},\frac{\nu_{\lfloor |S_{k}|/2 \rfloor+1}-\nu_1}{2}\big\}\big),
\end{align*}
where the fourth line follows from $\lfloor |S_{k}|/2 \rfloor \geq  |S_{k}|/2 -1$.

First we show that $\ell_{i,t}-\ell_{1,t} > 0$ for all $t \geq \tau_i := \gamma^{-1}\big(\frac{\nu_i-\nu_1}{2}\big)$.
Given the definition of $\gamma$, we have for all $i \in [n]$ that $|\ell_{i,t} - \nu_i| \leq \gamma(t) \leq \frac{\nu_i-\nu_1}{2}$ where the last inequality holds for $t \geq \tau_i$. 
Thus, for $t \geq \tau_i$ we have
\begin{align*}
    \ell_{i,t} - \ell_{1,t} &= \ell_{i,t} - \nu_i +\nu_i -\nu_1 + \nu_1 - \ell_{1,t} \\
    & = \ell_{i,t} - \nu_i - (\ell_{1,t} - \nu_1) + \nu_i - \nu_1 \\
    & \geq -2 \gamma(t) +\nu_i - \nu_1 \\
    & \geq -2\frac{\nu_i-\nu_1}{2} +\nu_i - \nu_1\\
    &= 0.
\end{align*}
Under this scenario, we will eliminate arm $i$ before arm $1$ since on each round the arms are sorted by their empirical losses and the top half are discarded.
Note that by the assumption the $\nu_i$ limits are non-decreasing in $i$ so that the $\tau_i$ values are non-increasing in $i$.

Fix a round $k$ and assume $1 \in S_k$ (note, $1\in S_0$).
The above calculation shows that  
\begin{align}
t \geq \tau_i \implies \ell_{i,t} \geq \ell_{1,t}. \label{eqn:suff_emp_losses}
\end{align} 
Consequently,
\begin{align*}
\{1 \in S_k, \ \ 1 \notin S_{k+1} \} \iff& \left\{ \sum_{i \in S_k} \1\{ \ell_{i,r_k} < \ell_{1,r_k} \} \geq \lfloor |S_{k}|/2 \rfloor \right\} \\
\implies& \left\{ \sum_{i \in S_k} \1\{ r_k < \tau_i \} \geq \lfloor |S_{k}|/2 \rfloor \right\} \\
\implies& \left\{ \sum_{i = 2}^{\lfloor |S_{k}|/2 \rfloor+1} \1\{ r_k < \tau_i \} \geq \lfloor |S_{k}|/2 \rfloor \right\} \\
\iff& \left\{ r_k < \tau_{\lfloor |S_{k}|/2 \rfloor+1} \right\}.
\end{align*}
where the first line follows by the definition of the algorithm, the second by Equation~\ref{eqn:suff_emp_losses}, and the third by $\tau_i$ being non-increasing (for all $i < j$ we have $\tau_i \geq \tau_j$ and consequently, $\1\{ r_k < \tau_i \} \geq \1\{ r_k < \tau_j \}$ so the \textit{first} indicators of the sum not including $1$ would be on before any other $i$'s in $S_k \subset [n]$ sprinkled throughout $[n]$).
This implies
\begin{align}
\{ 1 \in S_k, \ \ r_k \geq \tau_{\lfloor |S_{k}|/2 \rfloor+1} \} \implies \{ 1 \in S_{k+1} \}. \label{eqn:recursive_keep}
\end{align}

Recalling that $r_k \geq \gamma^{-1}\big(\max\big\{\frac{\epsilon}{4},\frac{\nu_{\lfloor |S_{k}|/2 \rfloor+1}-\nu_1}{2}\big\}\big)$ and $\tau_{\lfloor |S_{k}|/2 \rfloor+1} =  \gamma^{-1}\big(\frac{\nu_{\lfloor |S_{k}|/2 \rfloor+1}-\nu_1}{2}\big)$, we examine the following three exhaustive cases:
\begin{itemize}
    \item \textbf{Case 1:} $\frac{\nu_{\lfloor |S_{k}|/2 \rfloor+1}-\nu_1}{2} \geq \frac{\epsilon}{4}$ and $1 \in S_k$
    
In this case, $r_k \geq \gamma^{-1}\big( \frac{\nu_{\lfloor |S_{k}|/2 \rfloor+1}-\nu_1}{2}\big) = \tau_{\lfloor |S_{k}|/2 \rfloor+1}$.
By Equation~\ref{eqn:recursive_keep} we have that $1 \in S_{k+1}$ since $1 \in S_k$.

\item \textbf{Case 2:} $\frac{\nu_{\lfloor |S_{k}|/2 \rfloor+1}-\nu_1}{2} < \frac{\epsilon}{4}$ and $1 \in S_k$

In this case $r_k \geq \gamma^{-1}\big(\frac{\epsilon}{4}\big)$ but $\gamma^{-1}\big(\frac{\epsilon}{4}\big) < \tau_{\lfloor |S_{k}|/2 \rfloor+1}$. 
Equation~\ref{eqn:recursive_keep} suggests that it may be possible for $1 \in S_k$ but $1 \notin S_{k+1}$.
On the good event that $1 \in S_{k+1}$, the algorithm continues and on the next round either case 1 or case 2 could be true. 
So assume $1 \notin S_{k+1}$.
Here we show that $\{1 \in S_k, \ \ 1 \notin S_{k+1}\} \implies \max_{i \in S_{k+1}} \nu_i \leq \nu_1 +\epsilon/2$.
Because $1 \in S_0$, this guarantees that \succhalv either exits with arm $\widehat{i} = 1$ or some arm $\widehat{i}$ satisfying $\nu_{\widehat{i}} \leq \nu_1 + \epsilon/2$.

Let $p = \min \{i \in [n] : \frac{\nu_i -\nu_1}{2} \geq \frac{\epsilon}{4}\}$.  Note that $p>\lfloor |S_{k}|/2 \rfloor+1$ by the criterion of the case and 
\[r_k \geq \gamma^{-1}\left(\frac{\epsilon}{4}\right) \geq \gamma^{-1}\left(\frac{\nu_i-\nu_1}{2}\right) = \tau_i, \quad \forall i \geq p.\]
Thus, by Equation~\ref{eqn:suff_emp_losses} ($t \geq \tau_i \implies \ell_{i,t} \geq \ell_{1,t}$) we have that arms $i\geq p$ would always have $\ell_{i,r_k} \geq \ell_{1,r_k}$ and be eliminated before or at the same time as arm $1$, presuming $1 \in S_k$. 
In conclusion, if arm $1$ is eliminated so that $1 \in S_k$ but $1 \notin S_{k+1}$ then $\max_{i \in S_{k+1}} \nu_i \leq \max_{i < p} \nu_i < \nu_1 +\epsilon/2$ by the definition of $p$.

\item \textbf{Case 3:} $1 \notin S_k$

Since $1 \in S_0$, there exists some $r <k$ such that $1 \in S_r$ and $1 \notin S_{r+1}$. 
For this $r$, only case 2 is possible since case 1 would proliferate $1 \in S_{r+1}$. 
However, under case 2, if $1 \notin S_{r+1}$ then $\max_{i \in S_{r+1}} \nu_i \leq \nu_1 +\epsilon/2$.

\end{itemize}

Because $1 \in S_0$, we either have that $1$ remains in $S_k$ (possibly alternating between cases 1 and 2) for all $k$ until the algorithm exits with the best arm $1$, or there exists some $k$ such that case 3 is true and the algorithm exits with an arm $\widehat{i}$ such that $\nu_{\widehat{i}} \leq \nu_1 + \epsilon/2$.
The proof is complete by noting that 
\begin{align*}|\ell_{\widehat{i},\lfloor \tfrac{ B/2 }{ \lceil \log_2(n) \rceil} \rfloor} - \nu_1| \leq |\ell_{\widehat{i},\lfloor \tfrac{ B/2 }{ \lceil \log_2(n) \rceil} \rfloor} - \nu_{\widehat{i}}| + |\nu_{\widehat{i}} - \nu_1| \leq \epsilon/4 + \epsilon/2 \leq \epsilon
\end{align*}
by the triangle inequality and because $B \geq 2 \lceil \log_2(n) \rceil \gamma^{-1}(\epsilon/4)$ by assumption.

The second, looser, but perhaps more interpretable form of $z_{SH}$ follows from the fact that $\gamma^{-1}(x)$ is non-increasing in $x$ so that
\begin{align*}
 \max_{i=2,\dots,n}  i \, \gamma^{-1} \left(    \max\left\{ \tfrac{\epsilon}{4} , \tfrac{\nu_{i}- \nu_{1}}{2} \right\} \right)
&\leq \sum_{i=1,\dots,n}  \gamma^{-1} \left(   \max\left\{ \tfrac{\epsilon}{4} , \tfrac{\nu_{i}- \nu_{1}}{2} \right\} \right) .
\end{align*}
\end{proof}

\subsection{Proof of Lemma~\ref{lem:gamma_sum_bound_rand_arms}}
\begin{proof}
Let $p_n = \frac{\log(2/\delta)}{n}$, $M = \gamma^{-1}\left( \tfrac{\epsilon}{16} \right)$, and $\mu =  \E[ \min\{ M, \gamma^{-1}\left( \tfrac{\nu_{i}- \nu_{*}}{4} \right) \}]$.
Define the events
\begin{align*}
\xi_1 &= \{  \nu_1 \leq F^{-1}(p_n) \} \\
\xi_2 &= \left\{  \sum_{i=1}^n \min\{ M, \gamma^{-1}\left( \tfrac{\nu_{i}- \nu_{*}}{4} \right) \} \leq n \mu + \sqrt{2 n \mu M \log(2/\delta)} + \tfrac{2}{3} M \log(2/\delta) \right\}
\end{align*}
Note that $\P(\xi_1^c) = \P( \min_{i=1,\dots,n} \nu_i > F^{-1}(p_n) ) = (1-p_n)^n \leq \exp(-n p_n) \leq \frac{\delta}{2}$.
Moreover, $\P(\xi_2^c) \leq \frac{\delta}{2}$ by Bernstein's inequality since
\begin{align*}
\E\left[ \min\{ M, \gamma^{-1}\left( \tfrac{\nu_{i}- \nu_{*}}{4} \right) \}^2 \right] \leq \E\left[ M \min\{ M, \gamma^{-1}\left( \tfrac{\nu_{i}- \nu_{*}}{4} \right) \} \right] = M \mu.
\end{align*}
Thus, $\P(\xi_1 \cap \xi_2) \geq 1-\delta$ so in what follows assume these events hold.

First we show that if $\nu_* \leq \nu_1 \leq F^{-1}(p_n)$, which we will refer to as equation $(*)$, then $\max\left\{ \tfrac{\epsilon}{4} , \tfrac{\nu_{i}- \nu_{1}}{2} \right\} \geq \max\left\{ \tfrac{\epsilon}{4}, \tfrac{\nu_{i}- \nu_{*}}{4} \right\}$. \\
\textbf{Case 1}: $\tfrac{\epsilon}{4} \leq \tfrac{\nu_{i}- \nu_{1}}{2}$ and $\epsilon \geq 4(F^{-1}(p_n) - \nu_*)$.
\begin{align*}
 \tfrac{\nu_{i}- \nu_{1}}{2} \stackrel{(*)}{\geq}  \tfrac{\nu_{i}-\nu_* + \nu_* - F^{-1}(p_n) }{2} = \tfrac{\nu_{i}-\nu_*}{4} + \tfrac{\nu_{i}-\nu_*}{4} - \tfrac{F^{-1}(p_n) -\nu_*}{2} \stackrel{(*)}{\geq}  \tfrac{\nu_{i}-\nu_*}{4} + \tfrac{\nu_{i}-\nu_1}{4} - \tfrac{F^{-1}(p_n) -\nu_*}{2} \stackrel{{\scriptscriptstyle \text{Case 1}}}{\geq}  \tfrac{\nu_{i}-\nu_*}{4}.
 \end{align*}
\textbf{Case 2}: $\tfrac{\epsilon}{4} > \tfrac{\nu_{i}- \nu_{1}}{2}$  and $\epsilon \geq 4(F^{-1}(p_n) - \nu_*)$.
\begin{align*}
\tfrac{\nu_i-\nu_*}{4} = \tfrac{\nu_i - \nu_1}{4} + \tfrac{\nu_1 - \nu_*}{4} \stackrel{{\scriptscriptstyle \text{Case 2}}}{<} \tfrac{\epsilon}{8} + \tfrac{\nu_1 - \nu_*}{4} \stackrel{(*)}{\leq} \tfrac{\epsilon}{8} + \tfrac{F^{-1}(p_n) - \nu_*}{4}  \stackrel{{\scriptscriptstyle \text{Case 2}}}{<} \tfrac{\epsilon}{4}
\end{align*}
which shows the desired result.

Consequently, for any $\epsilon \geq 4(F^{-1}(p_n) - \nu_*)$ we have
\begin{align*}
\sum_{i=1}^n &\gamma^{-1} \left( \max\left\{ \tfrac{\epsilon}{4} , \tfrac{\nu_{i}- \nu_{1}}{2} \right\} \right) \leq \sum_{i=1}^n \gamma^{-1}\left( \max\left\{ \tfrac{\epsilon}{4}, \tfrac{\nu_{i}- \nu_{*}}{4} \right\} \right) \\
&\leq\sum_{i=1}^n \gamma^{-1}\left( \max\left\{ \tfrac{\epsilon}{16}, \tfrac{\nu_{i}- \nu_{*}}{4} \right\} \right) \\
&= \sum_{i=1}^n \min\{ M, \gamma^{-1}\left( \tfrac{\nu_{i}- \nu_{*}}{4} \right) \} \\
&\leq n \mu + \sqrt{2 n \mu M \log(1/\delta)} + \tfrac{2}{3} M \log(1/\delta) \\
&\leq \left( \sqrt{n \mu} + \sqrt{\tfrac{2}{3} M \log(2/\delta)} \right)^2 \leq 2n \mu + \tfrac{4}{3} M \log(2/\delta).
\end{align*}
A direct computation yields
\begin{align*}
\mu &=  \E[ \min\{ M, \gamma^{-1}\left( \tfrac{\nu_{i}- \nu_{*}}{4} \right) \}] \\
&= \E[ \gamma^{-1}\left( \max\left\{ \tfrac{\epsilon}{16}, \tfrac{\nu_{i}- \nu_{*}}{4} \right\} \right) ] \\
&= \gamma^{-1}\left( \tfrac{\epsilon}{16}\right) F( \nu_*+\epsilon /4) + \int_{\nu_*+\epsilon/4}^\infty \gamma^{-1}(\tfrac{t - \nu_{*}}{4}) dF(t) 
\end{align*}
so that
\begin{align*}
\sum_{i=1}^n &\gamma^{-1} \left( \max\left\{ \tfrac{\epsilon}{4} , \tfrac{\nu_{i}- \nu_{1}}{2} \right\} \right)  \leq 2n \mu + \tfrac{4}{3} M \log(2/\delta) \\
&= 2n  \int_{ \nu_*+\epsilon/4 }^\infty \gamma^{-1}(\tfrac{t - \nu_{*}}{4}) dF(t)  + \left(\tfrac{4}{3} \log(2/\delta) + 2n F(\nu_*+\epsilon/4) \right) \gamma^{-1}\left( \tfrac{\epsilon}{16} \right) 
\end{align*}
which completes the proof.
\end{proof}

\subsection{Proof of Proposition~\ref{prop:uniform_sampling_bound}}
We break the proposition up into upper and lower bounds and prove them seperately.

\subsection{Uniform Allocation}
\label{sec:lowerbound_proof}

\begin{proposition} 
Suppose we draw $n$ random configurations from $F$, train each with a budget of
$j$,\footnote{Here $j$ can be bounded (finite horizon) or unbounded (infinite horizon).} and let $\hat{\imath} = \arg\min_{i
=1,\dots,n} \ell_{j}(X_i)$. Let $\nu_i = \ell_*(X_i)$ and without loss
of generality assume $\nu_1 \leq \ldots \leq \nu_n$. If
\begin{align}
B \geq n \gamma^{-1}\left( \tfrac{1}{2} (F^{-1}(\tfrac{\log(1/\delta)}{n})  - \nu_*) \right) 
\end{align}
then with probability at least $1-\delta$ we have $\nu_{\hat{\imath}} - \nu_* \leq 2 \left( F^{-1}\left( \tfrac{ \log(1/\delta)}{n} \right) - \nu_* \right)$. 
\end{proposition}
\begin{proof}
Note that if we draw $n$ random configurations
from $F$ and $i_* = \arg\min_{i=1,\dots,n} \ell_*(X_i)$ then
\begin{align*}
\P\left( \ell_*(X_{i_*}) -\nu_* \leq \epsilon \right) &= \P\Big(
   \bigcup_{i=1}^n \{  \ell_*(X_i) -\nu_* \leq \epsilon \} \Big) \\
&= 1 - (1 - F(\nu_* + \epsilon))^n \geq 1- e^{-n  F(\nu_* + \epsilon)}
\,,
\end{align*}
which is equivalent to saying that with probability at least $1 -
\delta$, $\ell_*(X_{i_*}) -\nu_* \leq F^{-1}(\log(1/\delta)/n) -
\nu_*$.
Furthermore, if each configuration is trained for $j$ iterations then
with probability at least $1-\delta$
\begin{align*}
\ell_{*}(X_{\hat{\imath}}) - \nu_*  \leq \ell_{j}(X_{\hat{\imath}}) - \nu_* + \gamma(j) \leq \ell_{j}(X_{i_*}) - \nu_* + \gamma(j)  \\
\leq \ell_{*}(X_{i_*}) - \nu_* + 2\gamma(j)  \leq F^{-1}\left( \tfrac{ \log(1/\delta) }{n} \right) - \nu_* +
2\gamma(j)  .
\end{align*}
If our measurement budget $B$ is constrained so that $B = n j$ then
solving for $j$ in terms of $B$ and $n$ yields the result.
\end{proof}

The following proposition demonstrates that the upper bound on the error of
the uniform allocation strategy in
Proposition~\ref{prop:uniform_sampling_bound} is in fact tight. That is,
for any distribution $F$ and function $\gamma$ there exists a loss
sequence that requires the budget described in \eqref{eq:unif_budget}
in order to avoid a loss of more than $\epsilon$ with high probability.

\begin{proposition}
Fix any $\delta \in (0,1)$ and $n \in \mathbb{N}$. For any $c \in
(0,1]$, let $\mathcal{F}_c$ denote the space of continuous cumulative
distribution functions $F$ satisfying\footnote{
Note that this condition is met whenever $F$ is convex. Moreover, if $F(\nu_*+\epsilon)=c_1^{-1} \epsilon^\beta$ then it is easy to verify that $c = 1 - 2^{-\beta} \geq \tfrac{1}{2} \min\{ 1, \beta \}$. 
}
 $\inf_{x\in
[\nu_*,1-\nu_*]}\inf_{\Delta \in [0,1-x]}
\frac{F(x+\Delta)-F(x+\Delta/2)}{F(x+\Delta)-F(x)} \geq c$. And let
$\Gamma$ denote the space of monotonically decreasing functions over
$\mathbb{N}$. For any $F \in \mathcal{F}_c$ and $\gamma \in \Gamma$
there exists a probability distribution $\mu$ over $\X$ and a sequence
of functions $\ell_j : \X \rightarrow \R$ $ \ \forall j \in
\mathbb{N}$ with $\ell_* := \lim_{j \rightarrow \infty} \ell_j$,
$\nu_* = \inf_{x \in \X} \ell_*(x)$ such that $\sup_{x \in \X}
|\ell_j(x)-\ell_*(x)| \leq \gamma(j)$ and $\P_\mu( \ell_*(X) - \nu_*
\leq \epsilon ) = F(\epsilon)$. Moreover, if $n$ configurations
$X_1,\dots,X_n$ are drawn from $\mu$ and $\hat{\imath} = \arg\min_{i
\in 1,\dots,n} \ell_{B/n}(X_i)$ then with probability at least
$\delta$
\begin{align*}
    \ell_*(X_{\hat{\imath}}) - \nu_* \geq 2(F^{-1}(\tfrac{\log(c/\delta)}{n+\log(c/\delta)}) - \nu_*)
\end{align*}
whenever $B \leq n \gamma^{-1}\left( 2 (F^{-1}(\tfrac{\log(c/\delta)}{n+\log(c/\delta)})  - \nu_*) \right)$.
\end{proposition}
\begin{proof}
Let $\X = [0,1]$, $\ell_*(x) = F^{-1}(x)$, and $\mu$ be the uniform
distribution over $[0,1]$. Define $\widehat{\nu} =
F^{-1}(\tfrac{\log(c/\delta)}{n+\log(c/\delta)})$ and set
\begin{align*}
    \ell_j(x) = \begin{cases} \widehat{\nu} + \tfrac{1}{2}\gamma(j) +
  (\widehat{\nu} + \tfrac{1}{2}\gamma(j) - \ell_*(x)) & \text{if } | \widehat{\nu} + \tfrac{1}{2}\gamma(j) - \ell_*(x) | \leq \tfrac{1}{2}\gamma(j) \\
    \ell_*(x) & \text{otherwise.}\end{cases}
\end{align*}
Essentially, if $\ell_*(x)$ is within $\tfrac{1}{2}\gamma(j)$ of
$\widehat{\nu} + \tfrac{1}{2} \gamma(j)$  then we set $\ell_j(x)$
equal to $\ell_*(x)$ reflected across the value $2 \widehat{\nu} +
\gamma(j)$. Clearly, $|\ell_j(x) - \ell_*(x)| \leq
\gamma(j)$ for all $x \in \X$. 

Since each $\ell_*(X_i)$ is distributed according to $F$, we have
\begin{align*}
\P\Big( \bigcap_{i=1}^n \{  \ell_*(X_i) -\nu_* \geq \epsilon \} \Big)
 = (1 - F(\nu_* + \epsilon))^n 
 \geq  e^{-n  F(\nu_* + \epsilon)/(1-F(\nu_* + \epsilon))} \,.
\end{align*}
Setting the right-hand-side greater than or equal to $\delta/c$ and solving for $\epsilon$, we find $\nu_* + \epsilon \geq F^{-1}(\tfrac{\log(c/\delta)}{n+\log(c/\delta)}) = \widehat{\nu}$.

Define $I_0 = [\nu_*,\widehat{\nu})$, $I_1 = [ \widehat{\nu},
\widehat{\nu} + \tfrac{1}{2} \gamma(B/n))$ and $I_2 = [ \widehat{\nu}
+ \tfrac{1}{2} \gamma(B/n), \widehat{\nu} + \gamma(B/n) ]$.
Furthermore, for $j \in \{0,1,2\}$ define $N_j = \sum_{i=1}^n
\1_{\ell_*(X_i) \in I_j}$. Given $N_0=0$ (which occurs with
probability at least $\delta/c$), if $N_1=0$ then
$\ell_*(X_{\hat{\imath}}) -\nu_* \geq
F^{-1}(\tfrac{\log(c/\delta)}{n+\log(c/\delta)}) + \tfrac{1}{2}
\gamma(B/n)$ and the claim is true. 

Below we will show that if $N_2 >0$ whenever $N_1>0$ then the claim is also true. We now show that this happens with at least probability $c$ whenever $N_1+N_2 = m$ for any $m >0$. Observe that
\begin{align*}
    \P( N_2 > 0  | N_1+N_2 = m) &= 1- \P( N_2 = 0  | N_1+N_2 = m) \\
    &= 1-(1-\P( \nu_i \in I_2 | \nu_i \in I_1 \cup I_2 ))^m \geq 1-(1-c)^m \geq c 
\end{align*}
since
\begin{align*}
	\P( \nu_i \in I_2 | \nu_i \in I_1 \cup I_2 ) &=  \frac{ \P( \nu_i \in I_2 ) }{ \P( \nu_i \in I_1 \cup I_2 ) } = \frac{ \P( \nu_i \in [\widehat{\nu}+\tfrac{1}{2} \gamma,
\widehat{\nu}+ \gamma ] ) }{ \P( \nu_i \in [\widehat{\nu},
\widehat{\nu}+ \gamma ] ) } 
    = \frac{F(\widehat{\nu}+ \gamma) - F(\widehat{\nu}+\tfrac{1}{2}
\gamma)}{ F(\widehat{\nu}+ \gamma) - F(\widehat{\nu}) }
    \geq c \,.
\end{align*}
Thus, the probability of the event that $N_0=0$ and $N_2 >0$ whenever $N_1>0$  occurs with probability at least $\delta/c \cdot c = \delta$, so assume this is the case in what follows.

Since $N_0 = 0$, for all $j\in \mathbb{N}$, each $X_i$ must fall into one of three cases:
\begin{enumerate}
\item  $\ell_*(X_i) > \widehat{\nu}+\gamma(j)
  \iff \ell_j(X_i) > \widehat{\nu} + \gamma(j)$
\item $\widehat{\nu} \leq \ell_*(X_i) < \widehat{\nu}+\tfrac{1}{2} \gamma(j)
  \iff \widehat{\nu} + \tfrac{1}{2} \gamma(j) < \ell_j(X_i) \leq \widehat{\nu} +  \gamma(j)$ 
\item $\widehat{\nu}+\tfrac{1}{2} \gamma(j) \leq \ell_*(X_i) \leq \widehat{\nu}+\gamma(j) 
  \iff \widehat{\nu} \leq \ell_j(X_i) \leq \widehat{\nu} + \tfrac{1}{2} \gamma(j)$ 
\end{enumerate}
The first case holds since within that regime we have $\ell_j(x) =
\ell_*(x)$, while the last two cases hold since they consider the regime where
$\ell_j(x) = 2 \widehat \nu + \gamma(j) - \ell_*(x)$. Thus, for any
$i$ such that $\ell_*(X_i) \in I_2$ it must be the case that
$\ell_j(X_i) \in I_1$ and vice versa.
Because $N_2 \geq N_1 > 0$, we conclude that if $\hat{\imath} = \arg\min_i \ell_{B/n}(X_i)$ then $ \ell_{B/n}(X_{\hat{\imath}}) \in I_1$ and $ \ell_{*}(X_{\hat{\imath}}) \in I_2$. That is, $\nu_{\hat{\imath}} -\nu_* \geq \widehat{\nu} -\nu_* + \tfrac{1}{2} \gamma(j) = F^{-1}(\tfrac{\log(c/\delta)}{n+\log(c/\delta)}) - \nu_* + \tfrac{1}{2} \gamma(j)$. So if we wish $\nu_{\hat{\imath}} - \nu_* \leq 2(F^{-1}(\tfrac{\log(c/\delta)}{n+\log(c/\delta)}) - \nu_*)$ with probability at least $\delta$ then we require $B/n =j \geq \gamma^{-1}\left( 2 (F^{-1}(\tfrac{\log(c/\delta)}{n+\log(c/\delta)})  - \nu_*) \right)$.
\end{proof}

\subsection{Proof of Theorem~\ref{thm:hyperband_infinite_main}}
\begin{proof}
\textbf{Step 1: Simplify $\mathbf{H}(F,\gamma,n,\delta)$. }
We begin by simplifying $\mathbf{H}(F,\gamma,n,\delta)$ in terms of just $n,\delta,\alpha,\beta$.
In what follows, we use a constant $c$ that may differ from one inequality to the next but remains an absolute constant that depends on $\alpha,\beta$ only. 
Let $p_n=\tfrac{\log(2/\delta)}{n}$ so that
\begin{align*}
\gamma^{-1}\left( \tfrac{F^{-1}(p_n) - \nu_*}{4} \right)\leq c \left(F^{-1}\left(  p_n \right) - \nu_*\right)^{-\alpha}\leq c \, p_n^{-\alpha/\beta} 
 \end{align*}
and
\begin{align*}
\int_{ p_n }^1 \gamma^{-1}(\tfrac{F^{-1}(t) - \nu_{*}}{4}) dt \leq c \int_{p_n}^{1}  t^{-\alpha/\beta} dt  \leq \begin{cases}
c \log(1/p_n) & \text{ if } \alpha = \beta \\
c \frac{1 - p_n^{1-\alpha/\beta}}{1-\alpha/\beta} &\text{ if } \alpha \neq \beta.
\end{cases}
\end{align*}
We conclude that
\begin{align*}
\mathbf{H}(F,\gamma,n,\delta) &= 2n  \int_{ p_n }^1 \gamma^{-1}(\tfrac{F^{-1}(t) - \nu_{*}}{4}) dt  + \tfrac{10}{3} \log(2/\delta)  \gamma^{-1}\left(\tfrac{F^{-1}(p_n) - \nu_*}{4} \right)  \\
&\leq c p_n^{-\alpha/\beta} \log(1/\delta)  + c n  \begin{cases}
\log(1/p_n) & \text{ if } \alpha = \beta \\
\frac{1 - p_n^{1-\alpha/\beta}}{1-\alpha/\beta} &\text{ if } \alpha \neq \beta.
\end{cases}
\end{align*}

\noindent\textbf{Step 2: Solve for $(B_{k,l},n_{k,l})$ in terms of $\Delta$.}
Fix $\Delta >0$. 
Our strategy is to describe $n_{k,l}$ in terms of $\Delta$.
In particular, parameterize $n_{k,l}$ such that $p_{n_{k,l}} = c \frac{\log(4 k^3/\delta)}{n_{k,l}} = \Delta^\beta$ so that $n_{k,l} = c \Delta^{-\beta} \log( 4k^3/\delta)$  so
\begin{align*}
 \mathbf{H}(F,\gamma,n_{k,l},\delta_{k,l}) &\leq c p_{n_{k,l}}^{-\alpha/\beta} \log(1/\delta_{k,l})  + c n_{k,l} \begin{cases}
\log(1/p_{n_{k,l}}) & \text{ if } \alpha = \beta \\
\frac{1 - p_{n_{k,l}}^{1-\alpha/\beta}}{1-\alpha/\beta} &\text{ if } \alpha \neq \beta.
\end{cases}\\
 &\leq c \log(k/\delta) \bigg[ \Delta^{-\alpha} +   \begin{cases}
\Delta^{-\beta} \log( \Delta^{-1} ) & \text{ if } \alpha=\beta   \\
  \frac{\Delta^{-\beta} - \Delta^{-\alpha}}{1-\alpha/\beta} &\text{ if } \alpha \neq \beta   \\
 \end{cases}\bigg]\\ 
 &\leq c \log(k/\delta) \min\{ \tfrac{1}{|1-\alpha/\beta|} , \log(\Delta^{-1}) \}  \Delta^{-\max\{\beta,\alpha\}}
\end{align*}
where the last line follows from
\begin{align*}
\Delta^{\max\{\beta,\alpha\}} \frac{\Delta^{-\beta} - \Delta^{-\alpha}}{1-\alpha/\beta} &= \beta \frac{\Delta^{\max\{0,\alpha-\beta\}} - \Delta^{\max\{0,\beta-\alpha\}}}{\beta-\alpha} \\
&= \beta\begin{cases}
\frac{1 - \Delta^{\beta-\alpha}}{\beta-\alpha} & \text{ if } \beta > \alpha \\
\frac{1 - \Delta^{\alpha-\beta}}{\alpha-\beta} & \text{ if } \beta < \alpha 
\end{cases} \leq c \min\{ \tfrac{1}{|1-\alpha/\beta|} , \log(\Delta^{-1}) \}.
\end{align*}
Using the upperbound $\lceil \log(n_{k,l})\rceil \leq c \log( \log(k/\delta) \Delta^{-1}) \leq c\log( \log(k/\delta) ) \log(\Delta^{-1})$ and letting $z_\Delta = \log( \Delta^{-1} )^2 \Delta^{-\max\{\beta,\alpha\}}$, we conclude that 
\begin{align*}
B_{k,l} &< \min\{  2^k : 2^k  > 4 \lceil \log(n_{k,l})\rceil \mathbf{H}(F,\gamma,n_{k,l},\delta_{k,l}) \} \\
&< \min\{  2^k : 2^k  >  c \log( k/\delta) \log( \log(k/\delta) ) z_\Delta \} \\
&\leq c z_\Delta \log( \log(z_\Delta)/\delta ) \log( \log( \log(z_\Delta)/\delta ) ) \\
&= c z_\Delta \overline\log( \log(z_\Delta)/\delta ) .
\end{align*} 
\noindent\textbf{Step 3: Count the total number of measurements.}
Moreover, the total number of measurements before $\hat{\imath}_{k,l}$ is output is upperbounded by
\begin{align*}
T = \sum_{i=1}^k \sum_{j=l}^i B_{i,j} \leq k \sum_{i=1}^k B_{i,1} &\leq 2k B_{k,1} = 2 B_{k,1} \log_2(B_{k,1})  
\end{align*}
where we have employed the so-called ``doubling trick'': $\sum_{i=1}^k B_{i,1} = \sum_{i=1}^k 2^i \leq 2^{k+1} = 2 B_{k,i}$.
Simplifying, 
\begin{align*}
T \leq c z_\Delta \overline\log( \log(z_\Delta)/\delta ) \overline\log(z_\Delta \log( \log(z_\Delta)/\delta )) \leq c \Delta^{-\max\{\beta,\alpha\}}\overline\log( \Delta^{-1})^3 \overline\log( \log(\Delta^{-1})/\delta )
\end{align*}
Solving for $\Delta$ in terms of $T$ obtains 
\begin{align*}
\Delta = c \left( \frac{ \overline\log(T)^3  \overline\log(\log(T)/\delta)  }{T} \right)^{1/\max\{\alpha,\beta\}}.
\end{align*}
Because the output arm is just the empirical best, there is some error associated with using the empirical estimate. 
The arm returned returned on round $(k,l)$ is pulled $\lfloor \tfrac{ 2^{k-1} }{ l } \rfloor \gtrsim B_{k,l}/\log(B_{k,l})$ times so the possible error is bounded by $\gamma(B_{k,l}/\log(B_{k,l})) \leq c \left( \frac{\log(B_{k,l})}{B_{k,l}} \right)^{1/\alpha} \leq c \left( \frac{\log(B)^2 \log(\log(B))}{B} \right)^{1/\alpha} $ which is dominated by the value of $\Delta$ solved for above.
\end{proof}

\subsection{Proof of Theorem~\ref{thm:hyperband_infinite_alt}}
\begin{proof}
\textbf{Step 1: Simplify $\mathbf{H}(F,\gamma,n,\delta,\epsilon)$. }
We begin by simplifying $\mathbf{H}(F,\gamma,n,\delta,\epsilon)$ in terms of just $n,\delta,\alpha,\beta$.
As before, we use a constant $c$ that may differ from one inequality to the next but remains an absolute constant.
Let $p_n=\tfrac{\log(2/\delta)}{n}$ 
First we solve for $\epsilon$ by noting that we identify the best arm if $\nu_{\hat{\imath}} - \nu_* < \Delta_2$. 
Thus, if $\nu_{\hat{\imath}} - \nu_* \leq \big(F^{-1}(p_n) - \nu_*\big) + \epsilon/2$ then we set 
\begin{align*}
\epsilon = \max\left\{  2\big(\Delta_2 - \big(F^{-1}(p_n) - \nu_*\big)\big) , 4 \big(F^{-1}(p_n) - \nu_*\big)\right\}
\end{align*}
so that 
\begin{align*}
\nu_{\hat{\imath}} - \nu_* \leq \max\left\{ 3 \big(F^{-1}(p_n) - \nu_*\big) , \Delta_2 \right\} = \Delta_{ \lceil \max\{2, c K p_n \} \rceil }.
\end{align*}
We treat the case when $3 \big(F^{-1}(p_n) - \nu_*\big)  \leq \Delta_2$ and the alternative separately.

First assume $3 \big(F^{-1}(p_n) - \nu_*\big) > \Delta_2$ so that $\epsilon = 4 \big(F^{-1}(p_n) - \nu_*\big)$ and $\mathbf{H}(F,\gamma,n,\delta,\epsilon) = \mathbf{H}(F,\gamma,n,\delta)$.
We also have
\begin{align*}
\gamma^{-1}\left(\tfrac{F^{-1}(p_n) - \nu_*}{4}\right)\leq c \left(F^{-1}\left(  p_n \right) - \nu_*\right)^{-\alpha}\leq c \, \Delta_{ \lceil p_n K \rceil }^{-\alpha} 
\end{align*}
and
\begin{align*}
\int_{ p_n }^1 \gamma^{-1}(\tfrac{F^{-1}(t) - \nu_{*}}{4}) dt = \int_{ F^{-1}(p_n) }^1 \gamma^{-1}(\tfrac{x - \nu_{*}}{4}) dF(x) \leq \frac{c}{K} \sum_{i=\lceil p_n K \rceil}^K \Delta_i^{-\alpha}
\end{align*}
so that
\begin{align*}
\mathbf{H}(F,\gamma,n,\delta) &= 2n  \int_{ p_n }^1 \gamma^{-1}(\tfrac{F^{-1}(t) - \nu_{*}}{4}) dt  + \tfrac{10}{3} \log(2/\delta)  \gamma^{-1}\left( \tfrac{F^{-1}(p_n) - \nu_* }{4}\right)  \\
&\leq c \Delta_{ \lceil p_n K \rceil }^{-\alpha}  \log(1/\delta)  + \frac{c n }{K} \sum_{i=\lceil p_n K \rceil}^K \Delta_i^{-\alpha}.
\end{align*}

Now consider the case when $3 \big(F^{-1}(p_n) - \nu_*\big)  \leq \Delta_2$.
In this case $F(\nu_*+\epsilon/4) = 1/K$, $\gamma^{-1}\left( \tfrac{\epsilon}{16} \right) \leq c \Delta_2^{-\alpha}$, and $\int_{ \nu_*+\epsilon/4 }^\infty \gamma^{-1}(\tfrac{t - \nu_{*}}{4}) dF(t) \leq c \sum_{i=2}^K \Delta_i^{-\alpha}$ so that
\begin{align*}
\mathbf{H}(F,\gamma,n,\delta,\epsilon) &= 2n  \int_{ \nu_*+\epsilon /4}^\infty \gamma^{-1}(\tfrac{t - \nu_{*}}{4}) dF(t)  + \left(\tfrac{4}{3} \log(2/\delta) + 2n F(\nu_*+\epsilon/4) \right) \gamma^{-1}\left( \tfrac{\epsilon}{16} \right) \\
&\leq c (\log(1/\delta) + n /K) \Delta_2^{-\alpha} + \frac{cn}{K} \sum_{i=2}^K \Delta_i^{-\alpha}.
\end{align*}

\noindent\textbf{Step 2: Solve for $(B_{k,l},n_{k,l})$ in terms of $\Delta$.}
Note there is no improvement possible once $p_{n_{k,l}} \leq 1/K$ since $3 \big(F^{-1}(1/K) - \nu_*\big)  \leq \Delta_2$.
That is, when $p_{n_{k,l}} \leq 1/K$ the algorithm has found the best arm but will continue to take samples indefinetely.
Thus, we only consider the case when $q=1/K$ and $q > 1/K$. 
Fix $\Delta >0$. 
Our strategy is to describe $n_{k,l}$ in terms of $q$.
In particular, parameterize $n_{k,l}$ such that $p_{n_{k,l}} = c \frac{\log(4 k^3/\delta)}{n_{k,l}} = q$ so that $n_{k,l} = c q^{-1} \log( 4k^3/\delta)$  so
\begin{align*}
 \mathbf{H}(F,\gamma,n_{k,l},\delta_{k,l},\epsilon_{k,l}) &\leq c  \begin{cases}
  (\log(1/\delta_{k,l}) + \tfrac{n_{k,l} }{K}) \Delta_2^{-\alpha} + \frac{n_{k,l}}{K} \sum_{i=2}^K \Delta_i^{-\alpha} &\text{ if } 5 \big(F^{-1}(p_{n_{k,l}}) - \nu_*\big)  \leq \Delta_2 \\
\Delta_{ \lceil p_{n_{k,l}} K \rceil }^{-\alpha}  \log(1/\delta_{k,l})  + \frac{ n_{k,l} }{K} \sum_{i=\lceil p_{n_{k,l}} K \rceil}^K \Delta_i^{-\alpha} & \text{ if otherwise}
\end{cases} \\
&\leq c \log(k/\delta) \begin{cases}
	\Delta_2^{-\alpha} + \sum_{i=2}^K \Delta_i^{-\alpha} &\text{ if } q = 1/K \\
\Delta_{ \lceil q K \rceil }^{-\alpha}  + \frac{ 1}{q K} \sum_{i=\lceil q K \rceil}^K \Delta_i^{-\alpha} & \text{ if $q>1/K$}.
\end{cases} \\
	&\leq c \log(k/\delta) \Delta_{ \lceil \max\{2,q K\} \rceil }^{-\alpha}  + \frac{ 1}{q K} \sum_{i=\lceil \max \{2,q K\} \rceil}^K \Delta_i^{-\alpha} 
\end{align*}
	Using the upperbound $\lceil \log(n_{k,l})\rceil \leq c \log( \log(k/\delta) q^{-1}) \leq c\log( \log(k/\delta) ) \log(q^{-1})$ and letting $z_q = \log( q^{-1} ) (\Delta_{ \lceil \max\{2,q K\} \rceil }^{-\alpha}  + \frac{ 1}{q K} \sum_{i=\lceil \max\{2,q K\} \rceil}^K \Delta_i^{-\alpha} )$, we
apply the exact sequence of steps as in the proof of Theorem~\ref{thm:hyperband_infinite_main} to obtain
\begin{align*}
T \leq c z_q \overline\log( \log(z_q)/\delta ) \overline\log(z_q \log( \log(z_q)/\delta ))
\end{align*}
Because the output arm is just the empirical best, there is some error associated with using the empirical estimate. 
The arm returned on round $(k,l)$ is pulled $\lceil \tfrac{ 2^{k-1} }{ l } \rceil \geq c B_{k,l}/\log(B_{k,l})$ times so the possible error is bounded by $\gamma(B_{k,l}/\log(B_{k,l})) \leq c \left( \frac{\log(B_{k,l})}{B_{k,l}} \right)^{1/\alpha} \leq c \left( \frac{\log(T)^2 \log(\log(T))}{T} \right)^{1/\alpha}$. 
This is dominated by $\Delta_{\lceil \max\{2,q K\} \rceil}$ for the value of $T$ prescribed by the above calculation, completing the proof.
\end{proof}

\subsection{Proof of Theorem~\ref{thm:succ_halv_finite}}
\begin{proof}
Let $s$ denote the index of the last stage, to be determined later. If $\widetilde{r}_k= R \eta^{k-s}$ and $\widetilde{n}_k = n \eta^{-k} $ so that $r_k = \lfloor \widetilde{r}_k \rfloor$ and $n_k = \lfloor \widetilde{n}_k \rfloor$ then
\begin{align*}
\sum_{k=0}^s n_k r_k \leq \sum_{k=0}^s \widetilde{n}_k \widetilde{r}_k = n R (s+1)  \eta^{-s} \leq  B
\end{align*}
since, by definition, $s = \min \{ t \in \mathbb{N} : nR(t+1)\eta^{-t} \leq B, t\leq \log_\eta(\min\{R,n\})\}$. 
It is straightforward to verify that $B \geq z_{SH}$ ensures that $r_0 \geq 1$ and $n_s \geq 1$.

The proof for Theorem~\ref{thm:succ_halving} holds here with a few modifications.  First, we derive a lower bound on the resource per arm $r_k$ per round
if $B\geq z_{SH}$ with generalized elimination rate $\eta$:
\begin{align*}
    r_k &\geq \frac{B}{|S_k| (\lfloor \log_\eta(n) \rfloor +1)}-1  \\
    &\geq \frac{\eta}{|S_k|}\max_{i=2,\dots,n} i\bigg(1+\min\big\{R,\gamma^{-1}\big(\max\big\{\frac{\epsilon}{4},\frac{\nu_i-\nu_1}{2}\big\}\big)\big\}\bigg) -1 \\
    &\geq \frac{\eta}{|S_{k}|} \ (\lfloor |S_{k}|/\eta \rfloor+1) \bigg(1+\min\big\{R,\gamma^{-1}\big(\max\big\{\frac{\epsilon}{4},\frac{\nu_{\lfloor |S_{k}|/2 \rfloor+1}-\nu_1}{2}\big\}\big)\big\}\bigg) -1 \\
    &\geq \bigg(1+\min\big\{R,\gamma^{-1}\big(\max\big\{\frac{\epsilon}{4},\frac{\nu_{\lfloor |S_{k}|/2 \rfloor+1}-\nu_1}{2}\big\}\big)\big\}\bigg) - 1\\
    &= \min\big\{R,\gamma^{-1}\big(\max\big\{\frac{\epsilon}{4},\frac{\nu_{\lfloor |S_{k}|/2 \rfloor+1}-\nu_1}{2}\big\}\big)\big\}.
\end{align*}
Also, note that $\gamma(R)=0$, hence if the minimum is ever active, $\ell_{i,R} = \nu_i$ and we know the true loss.  
The rest of the proof is same as that for Theorem~\ref{thm:succ_halving} for $\eta$ in place of 2.

In addition, we note that
\begin{align*}
\max_{i=n_s+1,\dots,n}  i \ \gamma^{-1} \left( \max\left\{ \tfrac{\epsilon}{4} , \tfrac{\nu_{ i }- \nu_{1}}{2} \right\} \right) \leq n_s \gamma^{-1} \left( \max\left\{ \tfrac{\epsilon}{4} , \tfrac{\nu_{ n_s +1}- \nu_{1}}{2} \right\} \right) + \sum_{i > n_s} \gamma^{-1} \left( \max\left\{ \tfrac{\epsilon}{4} , \tfrac{\nu_{ i }- \nu_{1}}{2} \right\} \right).
\end{align*}
\end{proof}

\clearpage
\bibliography{hyperband}
\end{document}